\PassOptionsToPackage{dvipsnames,table}{xcolor}
\documentclass{article}

\usepackage[utf8]{inputenc}
\usepackage[T1]{fontenc}
\usepackage{lmodern}
\usepackage{microtype}
\usepackage{natbib}
\usepackage{amsmath, amssymb, amsthm}
\usepackage{booktabs}
\usepackage{graphicx}
\usepackage{caption}
\usepackage{subcaption}
\usepackage{wrapfig}
\usepackage{multirow}
\usepackage{hyperref}
\hypersetup{colorlinks=true,linkcolor=MidnightBlue,citecolor=MidnightBlue,urlcolor=MidnightBlue}
\usepackage{url}
\usepackage{algorithm}
\usepackage{algorithmic}
\usepackage{enumitem}
\setlist[itemize]{topsep=2pt,itemsep=1pt,parsep=0pt,partopsep=0pt}
\setlist[enumerate]{topsep=2pt,itemsep=1pt,parsep=0pt,partopsep=0pt}
\usepackage{minted}

\usepackage[preprint]{icml2025/icml2025}

\usepackage[capitalize,noabbrev]{cleveref}
\usepackage{xspace}
\crefname{equation}{}{}  
\crefname{theorem}{Theorem}{Theorems}
\Crefname{theorem}{Theorem}{Theorems}
\crefname{proposition}{Proposition}{Propositions}
\Crefname{proposition}{Proposition}{Propositions}

\def\*#1{\mathbf{#1}}
\newcommand{\RR}{\mathbb{R}}
\newcommand{\E}{\mathbb{E}}
\DeclareMathOperator{\mathspan}{span}
\DeclareMathOperator{\vecop}{vec}
\DeclareMathOperator{\rank}{rank}
\DeclareMathOperator*{\argmin}{arg\,min}
\DeclareMathOperator{\tr}{tr}

\DeclareRobustCommand{\psilora}{\textsc{PSI-LoRA}\xspace}

\DeclareRobustCommand{\svdlora}{\textsc{SVDLoRA}\xspace}
\DeclareRobustCommand{\lorsum}{\textsc{LoRSum}\xspace}
\DeclareRobustCommand{\scaledlorsum}{\textsc{F-LoRSum}\xspace}
\newcommand{\metricpow}{\gamma}

\usepackage[colorinlistoftodos,textsize=scriptsize]{todonotes}

\theoremstyle{plain}
\newtheorem{theorem}{Theorem}[section]
\newtheorem{proposition}{Proposition}
\newtheorem{remark}[theorem]{Remark}

\icmltitlerunning{Beyond SGD, Without SVD}

\begin{document}

\twocolumn[
\icmltitle{Beyond SGD, Without SVD:\\ Proximal Subspace Iteration LoRA with Diagonal Fractional K-FAC}

\icmlsetsymbol{equal}{*}

\begin{icmlauthorlist}
\icmlauthor{Abdulla Jasem Almansoori}{mbzuai}
\icmlauthor{Maria Ivanova}{ysda}
\icmlauthor{Andrey Veprikov}{mbzuai,sbai,mirai,brain}
\icmlauthor{Aleksandr Beznosikov}{mirai,brain}
\icmlauthor{Samuel Horv\'{a}th}{mbzuai}
\icmlauthor{Martin Tak\'{a}\v{c}}{mbzuai}
\end{icmlauthorlist}

\icmlaffiliation{mbzuai}{MBZUAI, Abu Dhabi, UAE}
\icmlaffiliation{ysda}{Yandex School of Data Analysis, Moscow, Russia}
\icmlaffiliation{sbai}{SB AI Lab, Moscow, Russia}
\icmlaffiliation{mirai}{Moscow Independent Research Institute of Artificial Intelligence (MIRAI), Moscow, Russia}
\icmlaffiliation{brain}{Basic Research of Artificial Intelligence Laboratory (BRAIn Lab), Moscow, Russia}

\icmlcorrespondingauthor{Abdulla Jasem Almansoori}{abdulla.almansoori@mbzuai.ac.ae}

\icmlkeywords{Optimization, LoRA, Low-rank, alternating least squares, SVD, K-FAC, Natural Gradient, Memory-efficient}

\vskip 0.3in
]

\printAffiliationsAndNotice{\icmlEqualContribution} 

\begin{abstract}
Low-Rank Adaptation (LoRA) fine-tunes large models by learning low-rank updates on top of frozen weights,
dramatically reducing trainable parameters and memory \citep{hu2022lora}.
In this work, we address the gap between training with full steps with low-rank projections (\svdlora) and LoRA fine-tuning.
We propose \lorsum, a memory-efficient subroutine that closes this gap for gradient descent
by casting LoRA optimization as a proximal sub-problem and solving it efficiently with alternating least squares updates,
which we prove to be an implicit block power method.
We recover several recently proposed preconditioning methods for LoRA as special cases,
and show that \lorsum can also be used for updating a low-rank momentum.
In order to address full steps with \emph{preconditioned} gradient descent,
we propose a scaled variant of \lorsum that uses structured metrics such as K-FAC and Shampoo,
and we show that storing the diagonal of these metrics still allows them to perform well while remaining memory-efficient.
Experiments on a synthetic task, CIFAR-100, and language-model fine-tuning on GLUE, SQuAD v2, and WikiText-103,
show that our method can match or improve LoRA baselines given modest compute overhead,
while avoiding full-matrix SVD projections and retaining LoRA-style parameter efficiency.\footnote{
    Code: \url{https://github.com/zeligism/PSI-LoRA}
}
\end{abstract}

\section{Introduction}
\label{sec:intro}
Fine-tuning large pretrained models remains costly
as updating all parameters demands substantial GPU memory and state tracking.
Parameter-efficient methods like LoRA replace full fine-tuning with low-rank adapters
$\*U\in\RR^{d_\text{out}\times r}$, $\*V\in\RR^{d_\text{in}\times r}$
added to a frozen base weight $\*W^\text{base}$ so that $\*W^\text{tuned} = \*W^\text{base} + \*U\*V^\top$
with $r\ll\min(d_\text{out},d_\text{in})$ \citep{hu2022lora}.
Although LoRA reduces parameter and memory costs, first-order training of $(\*U,\*V)$
can be ill-conditioned \citep{tong2021accelerating} and sensitive to hyperparameters.
Preconditioning with $\*V^\top\*V$ and $\*U^\top\*U$ remedies part of this
\citep{zhang2024riemannian,zhang2023preconditioned,zhang2021preconditioned},
yet each step still acts only locally on the current factor subspaces.

As an oracle baseline, consider taking a full step in weight space $\*W \gets \*W + \Delta$,
and then projecting it back to rank $r$ w.r.t.\ the Frobenius norm.
We refer to this baseline as \svdlora because by Eckart-Young-Mirsky theorem,
the optimal rank-$r$ projection is given by the truncated SVD of $\*W + \Delta$.
While this does not guarantee maximal objective decrease for a non-quadratic loss,
it cleanly isolates the optimization gap introduced by enforcing rank $r$ after a full update.
Indeed, this has been recently proposed as an initialization scheme for LoRA-One \citep{zhang2025one},
so one can imagine that reapplying this projection each step would yield the best possible LoRA performance.
But this is clearly costly in terms of both compute and memory
since it requires running an SVD every step on a $d_\text{out} \times d_\text{in}$ matrix.
Nonetheless, the existence of the rank-$r$ updates themselves---the difference between the current weights
and their rank-$r$ SVD projection---suggests that we can do better.

In this work, we propose \lorsum (Low-Rank Sum), a stable and memory-efficient subroutine
for finding the best rank-$r$ solution to a sum of low-rank matrices.
The utility of \lorsum lies in its ability to approximate the above SVD baseline without forming or operating on the full matrix,
which is inspired from the observation that the gradient $\*G$ in LoRA admits a memory-efficient factorization
based on the input and output gradients of the layer.
The derivation of \lorsum is simple and follows from casting the SVD of an update as a proximal sub-problem in LoRA parameter space.
The solution uses alternating least squares (ALS) to approximate the solution iteratively,
which we prove to be a block power method implicitly.
In addition to approximating the full update,
\lorsum can also be used for updating low-rank approximations of the momentum buffers efficiently.
We further show that \lorsum can recover many existing methods in the literature
\citep{tong2021accelerating,zhang2021preconditioned,zhang2023preconditioned,xu2023power,zhang2024riemannian} as special cases.

When the step $\Delta$ is not expressible as a sum of low-rank terms,
such as in preconditioned gradient descent like Adam \citep{kingma2017adammethodstochasticoptimization}
or natural gradient descent \citep{amari1998naturalgradientmethod},
the standard \lorsum is no longer directly applicable.
In order to address this important case, we propose \scaledlorsum (Fisher-metric \lorsum),
which is a memory-efficient scaled variant that uses Kronecker-factored metrics
like K-FAC \citep{martens2020optimizingneuralnetworkskroneckerfactored}
and Shampoo \citep{gupta2018shampoopreconditionedstochastictensor} to scale the updates.
In general, Kronecker-factored metrics are not memory-efficient for LoRA training since they require storing
full $d_\text{in}\times d_\text{in}$ and $d_\text{out}\times d_\text{out}$ matrices.
However, our implementation only requires storing their diagonal approximations,
which ultimately makes \scaledlorsum's memory storage requirements even less than Adam's.
Indeed, \scaledlorsum with rank-$1$ momentum and diagonal metrics needs an extra memory of $2 (d_\text{in} + d_\text{out})$,
whereas Adam requires an extra memory of $2r (d_\text{in} + d_\text{out})$ for the first and second moments.
Compare with the memory-efficient AdaFactor \citep{shazeer2018adafactoradaptivelearningrates} which requires
$(d_\text{in} + d_\text{out})$ extra memory for the second moment.
Still, we show in the experiments that, not only does \scaledlorsum remain memory-efficient,
it can even outperform full-tuned Adam baseline on some GLUE tasks \citep{wang2019gluemultitaskbenchmarkanalysis}.

\paragraph{Contributions.}
We make the following contributions:
\begin{enumerate}
    \item \lorsum: a stable and memory-efficient subroutine for approximating the best proximal rank-$r$ sum of low-rank matrices,
    which we used to approximate the SVD a full update and full momentum. We show that \lorsum recovers existing preconditioned LoRA methods as special cases,
    and prove in \cref{thm:subspace-iteration} that \lorsum is a subspace iteration method
    that converges to \svdlora as the number of inner iterations $K$ increases, as shown in \cref{fig:linear}.
    \item \scaledlorsum: a \lorsum variant for preconditioned steps with Kronecker-factored metrics.
    \item \psilora and Scaled \psilora, two practical and memory-efficient optimizers based on the subroutines \lorsum and \scaledlorsum, with full reference implementations in \cref{app:algo}.
    \item Show empirical gains and learning rate robustness over LoRA baselines on linear synthetic task, CIFAR-100, GLUE, SQuAD v2, and WikiText-103.
\end{enumerate}

\paragraph{Setup and notation.}
For a single layer with frozen $\*W^\text{base} \in \RR^{d_\text{out}\times d_\text{in}}$,
we learn rank-$r$ adapters $\*U\in\RR^{d_\text{out}\times r}$ and $\*V\in\RR^{d_\text{in}\times r}$,
so that the effective tuned weight at iteration $t$ is $\*W_t^\text{tuned} = \*W^\text{base} + \*U_t \*V_t^\top$.
Henceforth, we make the dependence on $\*W^\text{base}$ implicit as it does not affect our analysis.
We define the low-rank matrix formed by the adapter as $\*W_t := \*U_t \*V_t^\top$.

Let $f(\*W_t)$ be the objective and define the full gradient
$\*G_t := \nabla_{\*W} f(\*W_t)$, which is also equal to the gradient w.r.t.\ $\*W_t^\text{tuned}$.
By the chain rule, $\*G_{\*U,t} = \nabla_{\*U} f(\*W_t) = \*G_t\*V_t$
and $\*G_{\*V,t} = \nabla_{\*V} f(\*W_t)=\*G_t^\top\*U_t$.
Fix a linear layer, and let $\*X_t \in \RR^{B\times d_\text{in}}$ be the inputs and
$\*S_t = \nabla_{(\*W\*X)^\top} f(\*W) \in \RR^{B\times d_\text{out}}$ be the gradients w.r.t.\ its outputs, where $B$ is the effective batch size.
Then, $\*G_t = \*S_t^\top\*X_t$. We exploit this factorized structure to perform only thin GEMM operations.

\section{Related Work}
\label{sec:related}

\noindent\textbf{Preconditioned optimization for fixed-rank parameterizations.}
Prior work improves the conditioning of low-rank factor optimization by preconditioning factor gradients,
leading to updates that depend on $\*V^\top\*V$ and $\*U^\top\*U$ \citep{tong2021accelerating,zhang2021preconditioned,zhang2023preconditioned,xu2023power}.
Riemannian preconditioned LoRA derives analogous updates from fixed-rank manifold geometry \citep{zhang2024riemannian}.
We instead view LoRA training as approximately solving the rank-$r$ projection of a full (possibly preconditioned) step.
One simultaneous iteration recovers these methods, while additional alternating iterations refine toward the best rank-$r$ direction.

\noindent\textbf{Projection and SVD-based low-rank updates.}
Projecting a dense update back to rank $r$ via truncated SVD (our \svdlora oracle)
is the optimal Frobenius-norm rank-$r$ projection, but it is typically too expensive for large layers.
This motivates iterative low-rank approximation (e.g., subspace iteration) and incremental SVD updates \citep{BRAND200620}.
Our \lorsum can be seen as a warm-started subspace-iteration approximation that uses only thin operations and exploits gradient factorization.
Related in spirit, GaLore projects gradients to a low-dimensional subspace \citep{zhao2024galore}.
We instead target the rank-$r$ projection of a \emph{full step}, with $K$ controlling projection accuracy.

\noindent\textbf{LoRA variants, reparameterizations, and initialization.}
A number of works modify LoRA or its training protocol to improve efficiency or generalization.
AdaLoRA adapts rank allocation using SVD-inspired criteria \citep{zhang2023adalora},
while DoRA reparameterizes weights into magnitude and direction and adapts both components \citep{liu2024dora}.
Other methods periodically merge adapters into the base weights or change the adapter schedule \citep{lialin2023relora}.
Several approaches focus on improving optimization through initialization or local geometry.
PiSSA initializes LoRA adapters using leading singular directions of pretrained weights \citep{meng2024pissa},
and LoRA-Pro uses a higher-dimensional tangent space to derive improved updates \citep{wang2024lorapro}.
Particularly close in motivation is LoRA-One. It uses an SVD-based initialization associated with projecting a full step (\svdlora)
and then applies preconditioned factor updates \citep{zhang2025one}.
Our contribution is complementary because rather than changing the adapter or relying on periodic SVD resets,
we provide a memory-efficient subroutine that approximates the rank-$r$ projected step itself.

\noindent\textbf{Optimizer state under rank constraints and Fisher-structured scaling.}
Maintaining optimizer state (e.g.\ momentum) in low-rank training is nontrivial
because naive momentum accumulates past gradients that generally do not lie in the current factor subspaces.
LoFT addresses this by projecting momentum-like quantities onto iterates' subspaces,
which can restrict momentum to joint subspaces across time \citep{tastan2025loftlowrankadaptationbehaves}.
Recent work also proposes maintaining explicit low-rank factorizations of momentum matrices \citep{mahdavinia2025lowrankmomentumfactorizationmemory}.
In our approach, the same \lorsum routine used for projected steps can also be used to maintain a low-rank estimate of momentum in Frobenius norm.
Moreover, we incorporate curvature information through Fisher/Kronecker-structured metrics.
K-FAC approximates the natural-gradient metric using Kronecker-factored Fisher blocks \citep{heskes2000onnaturallearning,martens2020optimizingneuralnetworkskroneckerfactored},
and Shampoo uses Kronecker-structured second moments with fractional matrix powers \citep{gupta2018shampoopreconditionedstochastictensor}.
We endow our projection subproblem with such metrics (and practical diagonal approximations), yielding scaled variants of \lorsum and \psilora.

\noindent\textbf{How our work is positioned.}
Overall, we provide a reusable, SVD-free projection subroutine for approximating full steps and low-rank state updates within standard LoRA fine-tuning.

\section{Without SVD: Proximal Subspace Iteration LoRA}
\label{sec:method}

In this section, we derive \lorsum,
a memory-efficient subroutine for approximating the best rank-$r$ projection of a full step,
and describe \psilora (Proximal Subspace Iteration LoRA),
a LoRA optimizer that uses \lorsum to approximate \svdlora steps and low-rank momentum.

\subsection{Projected gradient view and the \svdlora oracle}
\label{sec:proj-gd-view}

Let $\*W_t := \*U_t \*V_t^\top$ denote the LoRA at the current step.
We can derive the ideal gradient descent in adapter space as a projected gradient step
\begin{equation}
    \label{eq:proj-gd}
    \begin{aligned}
    \*W_{t+1}^\mathrm{svd}
    &=
    \argmin_{\rank(\*W) \leq r}\,
    \langle \*G_t,\, \*W \rangle
    + \frac{1}{2 \eta} \| \*W - \*W_t \|_\mathrm{F}^2
    \\
    &=
    \argmin_{\rank(\*W) \leq r}\,
    \frac{1}{2} \| \*W - (\*W_t - \eta \*G_t) \|_\mathrm{F}^2
    \\
    &=
    \Pi_r(\*W_t - \eta \*G_t),
    \end{aligned}
\end{equation}
where $\Pi_r(\cdot)$ denotes the best rank-$r$ approximation in Frobenius norm.
By Eckart-Young-Mirsky, $\Pi_r(\*W)$ is given by the rank-$r$ truncated SVD of $\*W$.
We refer to \cref{eq:proj-gd} as the \svdlora baseline
and consider steps taken by any optimizer in general, e.g., Adam or momentum SGD.

\begin{proposition}[SVDLoRA projection of a full step]
    \label{prop:best-rank-r-step}
    For any matrix $\bar{\*W}_{t+1} := \*W_t + \Delta$, the Frobenius-norm best rank-$r$ approximation
    $\Pi_r(\bar{\*W}_{t+1}) \in \argmin_{\rank(\*W)\le r}\|\*W-\bar{\*W}_{t+1}\|_\mathrm{F}$
    is given by the rank-$r$ truncated SVD of $\bar{\*W}_{t+1}$.
\end{proposition}

We know that $\*W_t$ is low-rank by construction,
so our goal is to approximate $\Pi_r(\*W_t + \Delta)$ given a step $\Delta$
without forming $\Delta$ as a dense $d_\text{out}\times d_\text{in}$ matrix or running a truncated SVD each step.

\subsection{Approximating rank-$r$ projections via warm-started ALS}
In the gradient descent case, $\bar{\*W}_{t+1} = \*W_t + \Delta$ admits a memory-efficient representation
since $\Delta = -\eta \*G_t = -\eta \*S_t^\top \*X_t$.
Thus, we can write $\bar{\*W}_{t+1} = \*U_t \*V_t^\top - \eta \*S_t^\top \*X_t$ as a sum of two low-rank matrices.

We now consider a \emph{regularized} low-rank projection of $\bar{\*W}_{t+1}$
the standard rank-$r$ factorized space $(\*U, \*V)$
\begin{equation}
    \min_{\*U,\*V}
    \quad
    \frac{1}{2} \left\|
        \*U \*V^\top - \bar{\*W}_{t+1}
    \right\|_\mathrm{F}^2
    + \mathcal{R}(\*U, \*V).
    \label{eq:main-subproblem}
\end{equation}
When $\mathcal{R} \equiv 0$, the optimal $\*U\*V^\top$ equals $\Pi_r(\bar{\*W}_{t+1})$.
The solution is not unique since $\*U\*V^\top = (\*U\*A)(\*V\*A^{-1})^\top$ for any invertible $\*A$.
Choosing the solution that minimizes $\frac{1}{2} (\| \*U \|_\mathrm{F}^2 + \| \*V \|_\mathrm{F}^2)$
gives the balanced factors $(\*U_r \*\Sigma_r^{1/2}, \*V_r \*\Sigma_r^{1/2})$
of the truncated SVD \citep{recht2010nuclearnormminimization},
which are unique (up to right orthogonal transformations).
Note that this is \emph{different} from setting
$\mathcal{R}(\*U, \*V) = \frac{1}{2} (\| \*U \|_\mathrm{F}^2 + \| \*V \|_\mathrm{F}^2)$ in \cref{eq:main-subproblem},
which leads to the singular value shrinkage operator \citep{cai2010singularvaluethresholding}.

We next introduce a ridge regularization on the factors to obtain a well-defined fixed-point map
and connect to existing preconditioned LoRA updates.
In practice we find ridge regularization alone can be unstable due to factor scaling,
so we ultimately use a proximal step for stability.

Thus, we consider \cref{eq:main-subproblem} with the following choice of regularizer
\begin{equation}
    \mathcal{R}^\text{fro}(\*U,\*V)
    = \frac{\lambda}{2} \big( \|\*U\|_\mathrm{F}^2 + \|\*V\|_\mathrm{F}^2 \big),
    \label{eq:fro-regularizer}
\end{equation}
where $\lambda > 0$.
The first-order conditions of \cref{eq:main-subproblem} yield the coupled equations
\begin{equation}
    \begin{aligned}
    \*U
    &= \bar{\*W}_{t+1}\*V(\*V^\top\*V+\lambda\*I)^{-1},
    \\
    \*V
    &= \bar{\*W}_{t+1}^\top\*U(\*U^\top\*U+\lambda\*I)^{-1}.
    \end{aligned}
    \label{eq:psilora-ideal-update}
\end{equation}
These define a fixed-point map that we approximately solve by a small number of inner iterations.

\paragraph{Gauss-Seidel (alternating) inner iterations.}
Initialize $\*U^{(0)}=\*U_t$, $\*V^{(0)}=\*V_t$ and for $k=0,\dots,K-1$ run the alternating updates\footnote{
    The order of the updates can have a difference depending on the conditioning of $\*U$ and $\*V$ (e.g., at LoRA initialization).
    In our experiments, we adopt the order as shown.
}
\begin{equation}
    \begin{aligned}
        \*U^{(k+1)} &\gets \bar{\*W}_{t+1}\*V^{(k)}\Big((\*V^{(k)})^\top\*V^{(k)}+\lambda\*I\Big)^{-1},
        \\
        \*V^{(k+1)} &\gets \bar{\*W}_{t+1}^\top\*U^{(k+1)}\Big((\*U^{(k+1)})^\top\*U^{(k+1)}+\lambda\*I\Big)^{-1}.
    \end{aligned}
    \label{eq:psilora-alternating-update}
\end{equation}
Similar iterations have been used for low-rank matrix completion \citep{hastie2015matrixcompletionlowranksvd}.
We can also consider a Jacobi variant that updates $\*U^{(k+1)}$ and $\*V^{(k+1)}$ from $(\*U^{(k)},\*V^{(k)})$.

\paragraph{Jacobi (simultaneous) iterations and preconditioned LoRA.}
Consider a single simultaneous fixed-point iteration uses $(\*U_t,\*V_t)$ on the right-hand side of \cref{eq:psilora-ideal-update}
\begin{equation}
    \begin{aligned}
    \*U_{t+1} &\gets \bar{\*W}_{t+1}\*V_t(\*V_t^\top\*V_t+\lambda\*I)^{-1},\\
    \*V_{t+1} &\gets \bar{\*W}_{t+1}^\top\*U_t(\*U_t^\top\*U_t+\lambda\*I)^{-1}.
    \end{aligned}
    \label{eq:psilora-simultaneous-update}
\end{equation}
When $\bar{\*W}_{t+1}=\*U_t \*V_t^\top -\eta\*G_t$ and $\lambda$ is a small damping term,
dropping the mild shrinkage factors yields the familiar preconditioned factor updates
\begin{equation}
    \begin{aligned}
    \*U_{t+1} &\approx \*U_t - \eta\,\*G_{\*U,t}\,(\*V_t^\top\*V_t+\lambda\*I)^{-1},
    \\
    \*V_{t+1} &\approx \*V_t - \eta\,\*G_{\*V,t}\,(\*U_t^\top\*U_t+\lambda\*I)^{-1},
    \end{aligned}
    \label{eq:psilora-approx-simultaneous-update}
\end{equation}
which corresponds to ScaledGD($\lambda$) \citep{tong2021accelerating,xu2023power}
and Riemannian preconditioned LoRA (with $\lambda$ used as practical damping) \citep{zhang2024riemannian}.
PrecGD uses a similar update but with adaptive damping instead \citep{zhang2021preconditioned},
but we do not consider it for simplicity.
However, see \cref{sec:gradient-clipping} for an interesting connection.

The form \cref{eq:psilora-approx-simultaneous-update} is particularly appealing and easy to implement
since $\*G_{\*U,t}$ and $\*G_{\*V,t}$ are already computed by automatic differentiation (autograd)
and it only requires inverting small $r \times r$ matrices.
In addition to the fact that \cref{eq:psilora-approx-simultaneous-update} is an approximation,
extending \cref{eq:psilora-simultaneous-update} to multiple inner iterations via autograd alone is not possible.
The full gradients $\*G_t = \*S_t^\top \*X_t$ are not stored by autograd,
so we need backward hooks to save $\*S_t$ and $\*X_t$ in order to compute quantities such as $\*G_t {\*V^{(k+1)}}$,
whereas a normal backward pass can only compute $\*G_t {\*V_{t}} = \*G_t {\*V^{(0)}}$.

\subsection{\lorsum as proximal rank-$r$ projections.}
In practice, we found that the factor regularization \cref{eq:fro-regularizer} can sometimes be unstable and leads to diverging updates.
In order to control the updates themselves, we consider a \emph{proximal} regularizer w.r.t.\ $(\*U_t, \*V_t)$ instead
\begin{equation}
    \mathcal{R}_\text{prox}(\*U, \*V)
    = \frac{\rho}{2} (
        \left\| \*U - \*U_t \right\|_\mathrm{F}^2
        + \left\| \*V - \*V_t\right\|_\mathrm{F}^2
    ).
    \label{eq:prox-regularizer}
\end{equation}
Then, the subproblem \cref{eq:main-subproblem} under the proximal regularizer becomes
\begin{equation}
    \begin{aligned}
        \argmin_{\*U, \*V} \
        &\tfrac{1}{2} \| {\*U \*V^\top} - ({\*U_t \*V_t^\top} - \eta \*G_t) \|^2_\mathrm{F}
        \\&\quad
        + \tfrac{\rho}{2} \| {\*U} - {\*U_t} \|^2_\mathrm{F}
        {+ \tfrac{\rho}{2} \| \*V - \*V_t \|^2_\mathrm{F} }
        \\
        =
        \argmin_{\*U, \*V} \
        &\langle \*U \*V^\top, \*G_t \rangle_\mathrm{F}
        + \tfrac{1}{2 \eta} \| {\*U \*V^\top} - {\*U_t \*V_t^\top} \|^2_\mathrm{F}
        \\&\quad
        + \tfrac{\rho}{2\eta} \| { \*U} - {\*U_t} \|^2_\mathrm{F}
        {+ \tfrac{\rho}{2\eta} \| \*V - \*V_t \|^2_\mathrm{F} }
        .
    \end{aligned}
    \label{eq:proximal-subproblem}
\end{equation}
This construction can offer two interpretations:
(i) a \emph{proximal rank-$r$ projection} of the full step $\*W_{t+1} = \*W_t - \eta \*G_t$,
where the proximal term penalizes large deviations in LoRA space from the current step $(\*U_t, \*V_t)$,
and (ii) a proximal update in a lifted space $(\*W, \*U, \*V)$ on the linearized objective
$f(\*W) \approx f(\*W_t) + \langle \*G_t, \*W - \*W_t \rangle$,
under the \emph{non-linear} constraint $\*W = \*U \*V^\top$.
See \cref{app:proximal-interpretation} for further discussion.

\paragraph{Proximal ALS updates.}
Given the proximal regularizer \cref{eq:prox-regularizer},
the Gauss-Seidel iterations that approximately solve \cref{eq:proximal-subproblem} are given by
\begin{equation}
    \begingroup\scriptsize
    \begin{aligned}
    \*U^{(k+1)}
    &\gets
    \big(
        \bar{\*W}_{t+1}\*V^{(k)} + \rho \, \*U_t
    \big)\Big(
        (\*V^{(k)})^\top\*V^{(k)} + \rho\*I
    \Big)^{-1},
    \\
    \*V^{(k+1)}
    &\gets
    \big(
        \bar{\*W}_{t+1}^\top\*U^{(k+1)} + \rho \, \*V_t
    \big)\Big(
        (\*U^{(k+1)})^\top\*U^{(k+1)} + \rho\*I
    \Big)^{-1}.
    \end{aligned}
    \endgroup
    \label{eq:proximal-alternating-updates}
\end{equation}
The additional proximal terms ($\rho \*U_t,\rho \*V_t$) can seem like a small improvement,
but we found them to be important for the stability of the updates.

\paragraph{Low-Rank Sum (\lorsum)}
More generally, suppose we are given the next step as a sum of low-rank matrices
\begin{equation*}
    \bar{\*W}_{t+1} = \sum_{j=1}^n c_j\,\*U_j\*V_j^\top,
\end{equation*}
and we want a rank-$r$ approximation represented by factors $(\*U,\*V)$, warm-started at $(\*U_1,\*V_1)$.
We define \lorsum as (approximately) solving the proximal projection
\begin{equation*}
    \min_{\*U,\*V}\;
    \frac{1}{2}\|\*U \*V^\top - \bar{\*W}_{t+1}\|_\mathrm{F}^2
    + \frac{\rho}{2} \Big(
        \|\*U - \*U_1 \|_\mathrm{F}^2 + \|\*V - \*V_1\|_\mathrm{F}^2
    \Big),
\end{equation*}
with $K$ Gauss-Seidel ALS iterations, yielding the same updates in \cref{eq:proximal-alternating-updates} but with
$\bar{\*W}_{t+1} = \sum_{j=1}^n c_j\,\*U_j\*V_j^\top$.
These updates define our \lorsum subroutine, written as
$\*W_{t+1} = \lorsum(\sum_{j=1}^{n} c_j \*U_{j,t} \*V_{j,t}^\top;\, K, \rho)$.
We show a simplified procedure in \cref{algo:lorsum-short}.
For the full details, please refer to \cref{app:algo}.

The careful reader may note that by stacking
$[\*U_{1,t+1} \ldots \*U_{n,t+1} ] [\*V_{1,t+1} \ldots \*V_{n,t+1} ]^\top$,
we can instead perform a low-rank modification to a previously-obtained SVD \citep{BRAND200620},
which requires one SVD at initialization and then thin operations afterwards.
However, the benefit of introducing \lorsum is that we can reuse the efficient subroutine in \cref{eq:psilora-alternating-update}
to sum low-rank matrices under two practical knobs:
$K$ for controlling approximation accuracy and $\rho$ for controlling update stability.
Both of these properties are desirable for low-rank updates to be fast and stable enough for deep learning applications.

\begin{algorithm}[ht!]
\caption{\lorsum with proximal regularization}
\begin{algorithmic}[1]
    \STATE \textbf{Input:} Low-rank terms $\{c_i, \*U_i, \*V_i\}_{i=1}^{n}$, number of inner iterations $K$, and proximal regularizer $\rho > 0$.
    \STATE $\*U_1^{(0)} \gets \*U_1, \quad \*V_1^{(0)} \gets \*V_1$.
    \FOR{$k = 0, \cdots, K-1$}
        \STATE Update $\*U_1^{(k+1)}$ and $\*V_1^{(k+1)}$ according to \cref{eq:proximal-alternating-updates} with
        $\bar{\*W} = \sum_{j=1}^{n} c_j \*U_j \*V_j^\top$.
    \ENDFOR
\end{algorithmic}
\label{algo:lorsum-short}
\end{algorithm}

\paragraph{Subspace iteration.}
Let $\mathcal{P}_\*X := \*X (\*X^\top \*X)^{\dagger} \*X^\top$ be the projection onto the column space of $\*X$,
where ${}^\dagger$ denotes the Moore-Penrose pseudoinverse.
Consider the fixed-point equations in \cref{eq:psilora-ideal-update} with $\lambda = 0$,
and observe that by right-multiplying $\*U$'s update by $\*V^\top$ and left-multiplying $\*V$'s update by $\*U^\top$,
we can write them as projections
\begin{equation}
    \*U \*V^\top
    =
    \bar{\*W}_{t+1} \mathcal{P}_{\*V}
    =
    \mathcal{P}_{\*U} \bar{\*W}_{t+1}
    .
    \label{eq:psilora-alternating-projections}
\end{equation}
Under the alternating updates in \cref{eq:proximal-alternating-updates} with $\rho = 0$
(or equivalently \cref{eq:psilora-alternating-update} with $\lambda = 0$),
these projections can be seen as block power iterations on
$\bar{\*W}_{t+1}^\top \bar{\*W}_{t+1}$ and $\bar{\*W}_{t+1} \bar{\*W}_{t+1}^\top$.

\begin{theorem}[\lorsum as a warm-started subspace iteration]
    \label{thm:subspace-iteration}
    Consider the iterations \cref{eq:proximal-alternating-updates},
    where $\bar{\*W}_{t+1} = \*U_t^{(0)} \*V_t^{(0)} - \eta \*G_t$ and $\rho = 0$.
    Assume $\mathspan(\*V^{(0)})$ has nonzero overlap with the top-$r$ right singular subspace of $\bar{\*W}_{t+1}$.
    Let $\sigma_1 \geq \ldots \sigma_{\min\{d_\text{in}, d_\text{out}\}}$ be the singular values
    of $\bar{\*W}_{t+1}$ and let $\*W_{t+1}^\mathrm{svd} = \Pi_r(\bar{\*W}_{t+1})$ be its best rank-$r$ approximation.
    Then, for some constant $C > 0$ that depends on the initialization $\*V_t^{(0)}$, we have that
    \begin{equation*}
        \| \*U_t^{(k)} \*V_t^{(k)} -  \*W_{t+1}^\mathrm{svd}\|_\mathrm{F}^2
        \leq C \left( \frac{\sigma_{r+1}}{\sigma_{r}} \right)^{2k}.
    \end{equation*}
\end{theorem}
The proof uses a standard argument for subspace iteration and is deferred to \cref{app:proofs} for brevity.

\subsection{Momentum with \lorsum}
Let $\mathcal{M}(\*G_\*U)$ and $\mathcal{M}(\*G_\*V)$ be the momentum buffers for the \emph{LoRA gradients},
and consider only $\mathcal{M}(\*G_\*U)$ for brevity.
A naive implementation is to accumulate the preconditioned gradients
\begin{equation*}
    \mathcal{M}_{t}^\text{naive}(\*G_\*U)
    := \*G_{\*U, t} (\*V_t^\top \*V_t)^{-1} + \alpha \mathcal{M}_{t-1}^\text{naive}(\*G_\*U).
\end{equation*}
This does not preserve the property in \cref{eq:psilora-alternating-projections} as the buffers contain stale inverse gram matrices.
An intuitive fix is to re-project the momentum buffers before preconditioning LoFT~\citep{tastan2025loftlowrankadaptationbehaves}.
\begin{align*}
    \mathcal{M}_{t}^\text{proj}(\*G_\*U)
    &:= \*G_{\*U, t} (\*V_t^\top \*V_t)^{-1}
    \\&\quad + \alpha \mathcal{M}_{t-1}^\text{proj}(\*G_\*U) \*V_{t-1}^\top \*V_t (\*V_t^\top \*V_t)^{-1}
    \\
    {\color{Brown} \mathcal{M}_{t}^\text{proj}(\*G_\*U) \*V_t^\top}
    &= (\*G_t + \alpha {\color{Brown} \mathcal{M}_{t-1}^\text{proj}(\*G_\*U) \*V_{t-1}^\top} )\mathcal{P}_{\*V_t}
    ,
\end{align*}
However, this iteratively projects the buffer onto the \emph{joint} subspace of all iterations.
In addition, it requires saving the previous iterates.

Instead of maintaining factor-space momentum like LoFT,
we propose to maintain momentum in the full weight space using \lorsum.
We initialize a zero \emph{rank-$r_m$ momentum matrix} as in LoRA
(to satisfy the warm-start condition in \cref{thm:subspace-iteration}),
and use the \lorsum subroutine to update it to get the best estimate under the Frobenius norm.
\begin{equation}
    \mathcal{M}_{t}^\text{lor} := \lorsum(\alpha \mathcal{M}_{t-1}^\text{lor}(\*G) + \*G_t).
\end{equation}
This has memory cost $O((d_\text{in} + d_\text{out}) r_m)$,
which is cheaper than saving previous iterates if we choose $r_m < r$.
We compare against LoFT's momentum mechanism in \cref{sec:exp} and show that our {\lorsum}-based momentum
can outperform \svdlora under large enough rank budgets.
We show in \cref{algo:psilora-step-short} the the \psilora step with \lorsum momentum.
More details are provided in \cref{app:algo}.

\begin{algorithm}[ht!]
\caption{\psilora step with \lorsum momentum}
\begin{algorithmic}[1]
    \STATE \textbf{Input:} LoRA weights $\*U,\*V$, batch size $B$, learning rate $\eta$, momentum $\alpha$,
    $K$ inner iterations, proximal regularizer $\rho > 0$.
    \STATE Save in forward: $\*X$ (inputs).
    \STATE Save in backward: $\*S$ (grads of outputs).
    \STATE $\*U \*V^\top \gets \lorsum(\*U\*V^\top - \eta \*S^\top \*X - \eta \alpha  \mathcal{M};\, K, \rho)$.
    \STATE $\mathcal{M} \gets \lorsum(\alpha \mathcal{M} + \*S^\top \*X;\, K, \rho)$.
\end{algorithmic}
\label{algo:psilora-step-short}
\end{algorithm}

\section{Beyond SGD: \scaledlorsum with Kronecker-factored metrics}
We now extend \lorsum to handle Kronecker-factored metrics.
Specifically, we consider the \emph{K-FAC} metric, which approximates the Fisher information matrix.
Thus, we call our method Fisher-metric \lorsum (\scaledlorsum).
The procedure is described in \cref{algo:scaled-psilora-step},
which implements a preconditioned \psilora step with \scaledlorsum.
For notational simplicity, \cref{algo:scaled-psilora-step} does not detail how to perform the multiplications efficiently.
We provide the full details \cref{app:algo}.

\subsection{\lorsum with Kronecker-factored metrics}
\label{sec:scaled-lorsum}
We can endow the proximal terms in \cref{eq:proximal-subproblem} with non-Euclidean metrics
that can recover K-FAC and Shampoo updates.
Let $\*D_\*U \succ 0$ and $\*D_\*V \succ 0$ define the weighted Frobenius norm
\begin{equation*}
    \| \*W \|_{\mathrm{UV}}^2
    := \langle \*W,\, \*D_\*U \*W \*D_\*V \rangle_\mathrm{F}.
\end{equation*}
Similarly, we can define $\| \*U \|_{\*U}^2 := \langle \*U, \*D_\*U \*U \rangle_\mathrm{F}$
and $\| \*V \|_{\*V}^2 := \langle \*V, \*D_\*V \*V \rangle_\mathrm{F}$.
By Cauchy-Schwarz, we have $\|\*U\*V^\top\|_{\mathrm{UV}} \le \|\*U\|_{\*U}\,\|\*V\|_{\*V}$,
so these norms are consistent.

To simplify derivations, we define the linear operator
\begin{equation*}
    D(\*W) := \*D_\*U \, \*W \, \*D_\*V,
\end{equation*}
so that $\|\*W\|_{\mathrm{UV}}^2 = \| D^{1/2}(\*W) \|_\mathrm{F}^2$.
Now, consider preconditioning a step $\Delta_t$,
e.g., $\Delta_t = -\eta (\*G_t + \alpha \mathcal{M}_{t-1})$,
so that
\begin{equation*}
    \tilde{\*W}_{t+1} = \*W_t + D^{-1}(\Delta_t) = \*W_t + \*D_\*U^{-1} \Delta_t \*D_\*V^{-1}.
\end{equation*}

The reason we call $\*D_\*U$ and $\*D_\*V$ \emph{Kronecker-factored metrics} is because of the following identity
\begin{equation*}
    \vecop(\*D_\*U \*W \*D_\*V)
    = (\*D_\*V^\top \otimes \*D_\*U) \vecop(\*W),
\end{equation*}
where $\otimes$ is the Kronecker product and $\vecop(\cdot)$ is the vectorization operator.
Thus, preconditioned gradient descent under the metric $\|\cdot\|_{\mathrm{UV}}$
is equivalent to standard gradient descent under the Kronecker-factored preconditioner
$\*D_\*V \otimes \*D_\*U$ for symmetric $\*D_\*U$ and $\*D_\*V$.
Indeed, the metric rank-$r$ projection of $\tilde{\*W}_{t+1}$ is
\begin{equation*}
    \Pi_r^{\mathrm{UV}}(\tilde{\*W}_{t+1})
    \in
    \argmin_{\rank(\*W)\le r}\|\*W - \tilde{\*W}_{t+1}\|_{\mathrm{UV}}^2,
\end{equation*}
and since multiplication by invertible matrices preserves rank, we directly get
\begin{equation*}
    \Pi_r^{\mathrm{UV}}(\tilde{\*W}_{t+1})
    =
    D^{-1/2}(\Pi_r(D^{1/2}(\tilde{\*W}_{t+1}))).
\end{equation*}
We can interpret metric projections as running \lorsum in whitened space
and then unwhitening the output factors.

Thus, we can write the effective proximal subproblem under the metric rank-$r$ projection as
\begin{equation}
    \begin{aligned}
    \min_{\*U,\*V} \quad
        &\langle \*U \*V^\top, \Delta_t \rangle_\mathrm{F}
        + \frac{1}{2} \|
            \*U \*V^\top - \*W_t
        \|^2_\text{UV}
        \\ &\qquad
        + \frac{\rho}{2} \|
            \*U - \*U_t
        \|^2_\text{U}
        + \frac{\rho}{2} \|
            \*V - \*V_t
        \|^2_\text{V}
    .
    \end{aligned}
    \label{eq:metric-proximal-subproblem}
\end{equation}
Note the the inner product is Euclidean.
Same as before, the first-order optimality conditions give fixed-point equations,
for which the Gauss-Seidel iterations will yield the following updates,
\begin{equation}
    \begin{aligned}
    \*U^{(k+1)}
    &=
    \big(
        \*U \*V^\top \*D_\*V \*V^{(k)}
        + \rho \*U
        + \*D_\*U^{-1} \Delta_t \*V^{(k)}
    \big)
    \\ &\quad \cdot
    \big(
    (\*V^{(k)})^\top \*D_\*V \*V^{(k)} + \rho \*I_r
    \big)^{-1},
    \\
    \*V^{(k+1)}
    &=
    \big(
        \*V \*U^\top \*D_\*U \*U^{(k+1)}
        + \rho \*V
        + \*D_\*V^{-1} \Delta_t^\top \*U^{(k+1)}
    \big)
    \\ &\quad \cdot
    \big(
    (\*U^{(k+1)})^\top \*D_\*U \*U^{(k+1)} + \rho \*I_r
    \big)^{-1}
    .
    \end{aligned}
    \label{eq:metric-psilora-alternating-update}
\end{equation}
We refer to this variant as \scaledlorsum (Fisher \lorsum) since we consider K-FAC,
which is an approximation of the Fisher information metric.

\paragraph{Diagonal K-FAC for memory-efficiency.}
Natural gradient descent \citep{amari1998naturalgradientmethod} is a powerful algorithm,
but it is often intractable for large models because computing the Fisher information matrix is prohibitively expensive.
Kronecker-factored approximations such as K-FAC \citep{martens2020optimizingneuralnetworkskroneckerfactored}
make this scaling practical and are widely used in practice.
The K-FAC statistics for $\*D_\*U$ and $\*D_\*V$ are given by
\begin{equation*}
    \begin{aligned}
    \*D_\*U
    &= \E[\*s_t \*s_t^\top] \approx \frac{1}{B} \*S_t^\top \*S_t,
    \\
    \*D_\*V
    &= \E[\*x_t \*x_t^\top] \approx \frac{1}{B} \*X_t^\top \*X_t,
    \end{aligned}
\end{equation*}
where $\*x_t$ is a random input vector and $\*s_t$ is the corresponding gradient w.r.t.\ the layer outputs.
For large models, storing dense $\*D_\*U$ and $\*D_\*V$ defeats the very purpose of using low-rank adapters.
We therefore use diagonal K-FAC metrics,
\begin{equation*}
    \*D_\*U \approx \mathrm{diag}(\*v_\*s),
    \quad
    \*D_\*V \approx \mathrm{diag}(\*v_\*x),
\end{equation*}
where $\*v_\*s\in\RR^{d_\text{out}}$ and $\*v_\*x\in\RR^{d_\text{in}}$
are exponential moving average (EMA) estimates of per-coordinate second moments computed from $\*S_t$ and $\*X_t$.
We can similarly use Shampoo-style statistics \citep{gupta2018shampoopreconditionedstochastictensor}
as we show in \cref{app:scaled-lorsum-algo}.

\paragraph{Fractional metrics.}
A nice property of Kronecker products is that they are closed under fractional powers,
i.e., $(\*D_\*V \otimes \*D_\*U)^\metricpow = \*D_\*V^\metricpow \otimes \*D_\*U^\metricpow$ for any $\metricpow > 0$.
In general, one can choose any $\metricpow \in (0,1]$, which interpolates between first-order updates and natural-gradient scaling.
Inspired by the success of Shampoo-style $\metricpow = 1/4$ powers \citep{shi2023pytorchshampoo},
we use a fractional power $\metricpow = \tfrac{1}{2}$ for K-FAC in our experiments.
In fact, this square-root K-FAC enjoys an interesting property:
the preconditioned gradient step has a Frobenius norm bounded by $\min\{\rank(\*S), \rank(\*X)\}$ up to damping.
This fact is easy to show, but we defer it to \cref{app:sqrt-kfac} for brevity.
We believe this can act as a form of soft gradient normalization,
which is likely the reason why we can use very large learning rates.

\begin{algorithm}[ht!]
\caption{Preconditioned \psilora step with diagonal K-FAC metrics}
\begin{algorithmic}[1]
    \STATE \textbf{Input:} LoRA weights $\*U,\*V$, batch size $B$, learning rate $\eta$, gradient and metric smoothing $(\beta_1, \beta_2)$, inner iterations $K$ , proximal regularizer $\rho > 0$, fractional metric power $\metricpow$, and scale damping $\delta$.
    \STATE Save during forward: $\*X \in \mathbb{R}^{B \times d_x}$ (inputs).
    \STATE Save during backward: $\*S \in \mathbb{R}^{B \times d_y}$ (gradient w.r.t.\ outputs).
    \STATE $\*D_\*V \gets \beta_2 \*D_\*V + (1 - \beta_2) \frac{1}{B} \mathrm{diag}(\*X^\top\*X)$,
    \STATE $\*D_\*U \gets \beta_2 \*D_\*U + (1 - \beta_2) \frac{1}{B} \mathrm{diag}(\*S^\top\*S)$.
    \STATE $\*U \*V^\top \gets \scaledlorsum(\*U\*V^\top - \eta (1 - \beta_1) \*G - \eta \beta_1  \mathcal{M}(\*G); \, (\*D_\*U + \delta \*I_{d_y})^{\metricpow}, (\*D_\*V + \delta \*I_{d_x})^{\metricpow};\, K, \rho)$.
    \STATE $\mathcal{M}(\*G) \gets \lorsum(\beta_1 \mathcal{M}(\*G) + (1 - \beta_1) \*G; \, K, \rho)$.
\end{algorithmic}
\label{algo:scaled-psilora-step-short}
\end{algorithm}

\section{Experiments}

\begin{table*}[ht!]
\centering
\caption{
    RoBERTa-base GLUE: best main metric (mean $\pm$ std).
    Learning rates $\eta$ is tuned over 3 values per method per task and method on ~20\% of the training budget.
    Then, the overall best $\eta$ is used to train on the full budget.
    See \cref{fig:glue_roberta} for full learning curves and $\eta$.
}
\label{tab:glue_roberta_best_metrics}
\begin{tabular}{lcccccc}
\toprule
Method & MNLI & QNLI & QQP & SST2 & COLA & STSB \\
\midrule
Full & \textbf{87.67{\scriptsize $\pm$ 0.19}} & \textbf{92.82{\scriptsize $\pm$ 0.06}} & \textbf{91.10{\scriptsize $\pm$ 0.06}} & \textbf{94.44{\scriptsize $\pm$ 0.08}} & \underline{62.00{\scriptsize $\pm$ 0.25}} & 90.51{\scriptsize $\pm$ 0.02} \\
SVDLoRA & 86.40{\scriptsize $\pm$ 0.48} & 92.13{\scriptsize $\pm$ 0.08} & 88.68{\scriptsize $\pm$ 0.03} & 94.09{\scriptsize $\pm$ 0.08} & \textbf{63.35{\scriptsize $\pm$ 0.65}} & \textbf{91.00{\scriptsize $\pm$ 0.03}} \\
\midrule
LoRA & 86.39{\scriptsize $\pm$ 0.04} & 91.70{\scriptsize $\pm$ 0.01} & 88.57{\scriptsize $\pm$ 0.08} & 93.75{\scriptsize $\pm$ 0.57} & 61.86{\scriptsize $\pm$ 0.73} & 90.56{\scriptsize $\pm$ 0.13} \\
RPLoRA & 85.67{\scriptsize $\pm$ 0.09} & 84.64{\scriptsize $\pm$ 8.91} & 87.06{\scriptsize $\pm$ 0.62} & 93.58{\scriptsize $\pm$ 0.00} & 58.74{\scriptsize $\pm$ 3.53} & 80.58{\scriptsize $\pm$ 2.17} \\
Proj. \psilora & 84.91{\scriptsize $\pm$ 0.03} & 90.46{\scriptsize $\pm$ 0.21} & 86.96{\scriptsize $\pm$ 0.11} & 93.35{\scriptsize $\pm$ 0.16} & 59.29{\scriptsize $\pm$ 0.17} & 89.93{\scriptsize $\pm$ 0.00} \\
Scaled \psilora & \underline{86.57{\scriptsize $\pm$ 0.06}} & \underline{92.18{\scriptsize $\pm$ 0.16}} & \underline{89.10{\scriptsize $\pm$ 0.07}} & \underline{94.15{\scriptsize $\pm$ 0.16}} & 61.41{\scriptsize $\pm$ 1.29} & \underline{90.83{\scriptsize $\pm$ 0.26}} \\
\bottomrule
\end{tabular}
\end{table*}

\label{sec:exp}
In all our experiments, linear layers are frozen and adapters are added.
The rest of the model parameters, if any, are trained normally.
We consider a linear task, image classification on CIFAR-100 \citep{krizhevsky2009learning},
GLUE text classification tasks \citep{wang2019gluemultitaskbenchmarkanalysis},
SQuAD v2 question answering tasks \citep{rajpurkar2018knowunanswerablequestions},
and language modeling on WikiText-103 \citep{merity2016pointer}.
We also ablate the parameters of Scaled \psilora on GLUE-MNLI to show its robustness.
All experiments use rank $r=8$ adapters, except CIFAR-100 which uses $r=16$.
We fine-tune query and value projection layers in GLUE and SQuAD tasks.
We implement all methods in PyTorch \citep{pytorch}
and use the Hugging Face's PEFT \citep{peft} and fine-tuning scripts as the base code for the language model LoRA experiments.

\textbf{Baselines.}
We consider the following tuning baselines:
Full-weight, LoRA \citep{hu2022lora},
Riemannian Preconditioned LoRA (RPLoRA) \citep{zhang2024riemannian,tong2021accelerating,zhang2021preconditioned},
\svdlora,
\psilora with \lorsum momentum (\cref{algo:psilora-step-short}),
\psilora with LoFT's projected momentum (Proj. \psilora) \citep{tastan2025loftlowrankadaptationbehaves},
and scaled \psilora with diagonal K-FAC metrics (\scaledlorsum).
We consider using both SGD or AdamW optimizers \citep{loshchilov2019decoupledweightdecay} for Full-weight, LoRA, RPLoRA, and \svdlora.
AdamW always uses the default parameters with zero weight decay unless otherwise stated.
For non-LoRA parameters, we use SGD with momentum for non-scaled \psilora and AdamW for Scaled \psilora.
Note that \svdlora maintains \emph{full-rank} optimization states, such as momentum or second moment in AdamW.
For \psilora and its scaled variant on CIFAR-100,
we consider different number of inner alternating iterations $K$ and use the notation "\psilora $\times K$".
For language model experiments, we consider simultaneous updates with $K=1$.

\textbf{Linear task.}
The linear task is as simple as \cref{eq:main-subproblem};
we consider the objective $\min_{\*U, \*V} \frac{1}{2} \| \*U \*V^\top - \*W \|_\mathrm{F}^2$
with $\*W \in \mathbb{R}^{600 \times 200}$.
The full-batch task uses SGD without momentum,
and mini-batch samples columns (batch size 64) and uses momentum $\alpha=0.75$.

In the full-batch case, \svdlora converges in one step with $\eta=1$, and we use the same $\eta$ for \psilora.
In \cref{fig:linear}, we show how \psilora converges to \svdlora as $K \to \infty$.
In the stochastic setting, \psilora outperforms \svdlora for small $K$,
and larger $K$ tracks \svdlora more closely but can plateau earlier.
This counter-intuitive behavior suggests that \svdlora's local optimality (\cref{prop:best-rank-r-step})
may not always translate to better generalization in stochastic settings,
despite maintaining full-rank momentum.

\begin{figure}[t]
    \centering
    \includegraphics[width=0.47\linewidth]{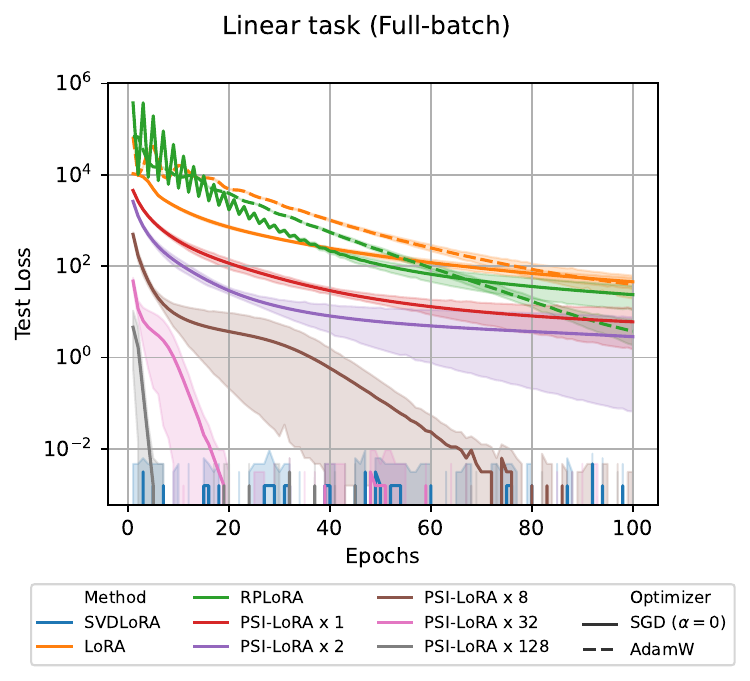}
    \includegraphics[width=0.52\linewidth]{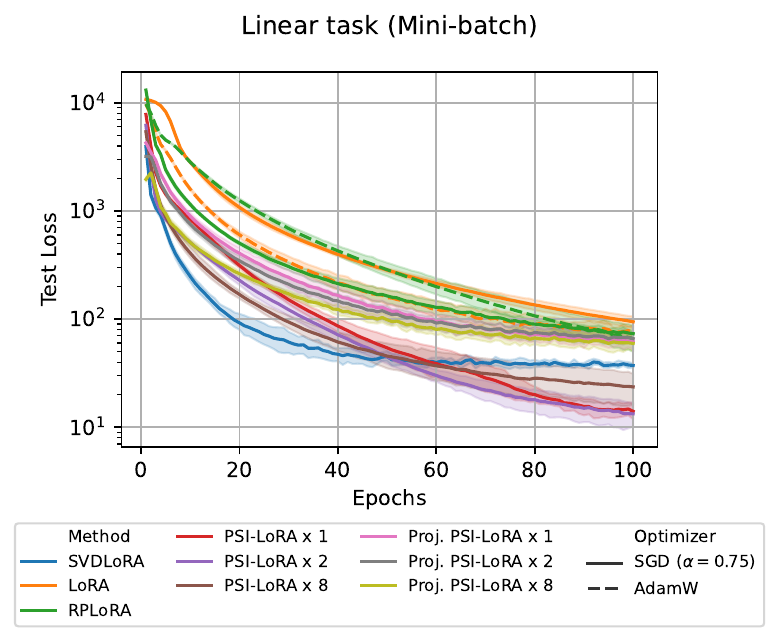}
    \caption{
        Linear task with full-batch gradients and no momentum (left) vs.\ mini-batch gradients with momentum (right).
        $\psilora \times K$ means $K$ alternating iterations.
    }
    \label{fig:linear}
\end{figure}

\textbf{CIFAR-100 task.}
For CIFAR-100, we use a custom MLP with one convolution layer that transforms the input into patches,
and then we process the patches using linear layers throughout.
Further, we initialize LoRAs with \svdlora of the base layers' weights, and then set the base layers to zero.
This task serves as a stress-test of LoRA methods and a proof of concept for \psilora.
We start noticing in this task that increasing the number of inner iterations $K$ does not necessarily improve performance.
Results are shown in \cref{tab:cifar100_results}.

\begin{table}
\centering
\caption{
    CIFAR-100 (PatchMLP, batch size 64):
    top test accuracy over time over 3 seeds.
    Bold is best, underlined is second best.}
\label{tab:cifar100_results}
\begingroup\small
\setlength{\tabcolsep}{5pt}
\renewcommand{\arraystretch}{1.05}
\begin{tabular}{lc}
\toprule
\textbf{Method} & \textbf{Top Accuracy} \\
\midrule
SVDLoRA (SGD) & 32.053{\scriptsize $\pm$ 0.179}  \\
Full (SGD) & \textbf{37.456}{\scriptsize $\pm$ 0.168}  \\
\midrule
LoRA (AdamW) & 31.393{\scriptsize $\pm$ 0.056} \\
RPLoRA (AdamW) & 32.895{\scriptsize $\pm$ 0.530} \\
\psilora x 1 & 30.143{\scriptsize $\pm$ 0.377}  \\
\psilora x 2 & 30.302{\scriptsize $\pm$ 0.305}  \\
\psilora x 8 & 30.472{\scriptsize $\pm$ 0.094} \\
Scaled \psilora x 1 & \underline{33.575}{\scriptsize $\pm$ 0.525} \\
Scaled \psilora x 2 & 32.739{\scriptsize $\pm$ 0.176} \\
Scaled \psilora x 8 & 31.678{\scriptsize $\pm$ 0.288}  \\
\bottomrule
\end{tabular}
\endgroup
\end{table}

\textbf{RoBERTa on GLUE.}
We demonstrate the feasibility of Scaled \psilora on a real-world sequence classification fine-tuning task.
Training is done via a tweaked fine-tuning script from Hugging Face.
We use RoBERTa \citep{liu2019robertarobustlyoptimizedbert}
for sequence classification on GLUE tasks \citep{wang2019gluemultitaskbenchmarkanalysis}.
We consider the tasks MNLI, QNLI, QQP, SST-2, CoLA, and STS-B.
We tune on ~20\% training budget (with a constant schedule) over 3 learning rates,
and pick the overall best performing hyper-parameters across all tasks.
The values in \cref{tab:glue_roberta_best_metrics} are reported for the same learning rate per method
(see \cref{fig:glue_roberta} for full learning curves and learning rates).
For example, for Scaled \psilora, we set
$\eta=0.2$, $\beta_1=0.9$, $\beta_2=0.99$, $\delta=10^{-5}$, $K=1$ (simultaneous), and $\rho=0.01$,
which we found to be fairly tranferable across other tasks as well,
where tuning $\beta_2$ can sometimes give further improvements.

The reason we set the same learning rate across tasks of varying sizes and difficulties
is to check the robustness of the optimizers and how much tuning is required in practice.
This is especially evident for CoLA, as the hyperparameters are not tuned specifically for it.
Also, RPLoRA's learning rate does not generalize well as can be seen from its large variance on QNLI, for example.

\textbf{T5 on SQuAD v2.}
We also show that Scaled \psilora perform well on a seq2seq task using T5 \citep{raffel2020exploringlimitstransfer}
fine-tuned on SQuAD v2 \citep{rajpurkar2018knowunanswerablequestions}.
We similarly tune over 3 learning rates per method and consider full training budget.
See \cref{app:exp} for more details.

\begin{table}[ht!]
\centering
\caption{
    T5-base SQuAD v2, tuned over 3 learning rates.
}
\label{tab:squad_t5_best_metrics}
\begingroup\small
\setlength{\tabcolsep}{4pt}
\renewcommand{\arraystretch}{1.05}
\begin{tabular}{lcccc}
\toprule
Method & $\eta$ & EM & F1 & Eval loss \\
\midrule
Ours & 0.5 & \textbf{60.80 {\scriptsize $\pm$ 0.9410}} & \textbf{62.41 {\scriptsize $\pm$ 1.02}} & \underline{0.7546 {\scriptsize $\pm$ 0.0019}} \\
Full & 1e-4 & \underline{56.57 {\scriptsize $\pm$ 3.64}} & 57.55 {\scriptsize $\pm$ 4.28} & \textbf{0.7145 {\scriptsize $\pm$ 0.0288}} \\
LoRA & 5e-4 & 56.52 {\scriptsize $\pm$ 1.60} & \underline{57.56 {\scriptsize $\pm$ 1.86}} & 0.7951 {\scriptsize $\pm$ 0.0144} \\
\bottomrule
\end{tabular}
\endgroup
\end{table}

\textbf{GPT-2 on WikiText-103.}
This is a language modeling task where we fine-tune GPT-2 \citep{radford2019language} on WikiText-103 \citep{merity2016pointer} with LoRA adapters.
We show that Scaled \psilora is much more robust than LoRA and SVDLoRA across a search over 3 learning rates, i.e., \{1e-4, 2e-4, 5e-4\} for LoRA and SVDLoRA, and \{0.5, 1.0, 2.0\} for Scaled \psilora.
Scaled \psilora achieves better perplexity no matter which learning rate we set.
The results are shown in \cref{fig:gpt2_wikitext}.

\begin{figure}[ht!]
    \captionsetup{font=small,justification=raggedright,singlelinecheck=false}
    \begin{minipage}[c]{0.62\linewidth}
        \centering
        \includegraphics[width=\linewidth]{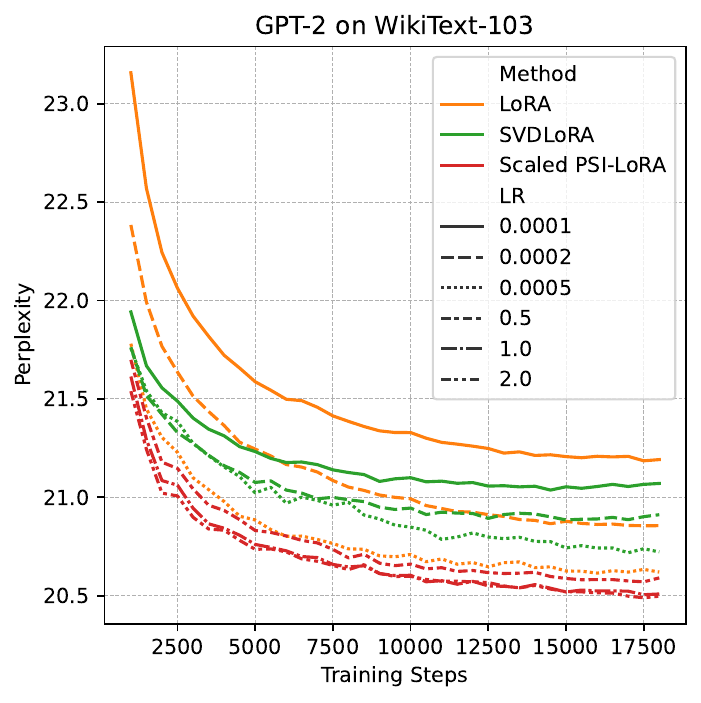}
    \end{minipage}\hfill
    \begin{minipage}[c]{0.35\linewidth}
        \caption{
            Perplexity of GPT-2 fine-tuning on WikiText-103 with LoRA adapters for 3 learning rates per method.
            Scaled \psilora is more robust to learning rate choice and achieves better perplexity for all learning rates.
        }
        \label{fig:gpt2_wikitext}
    \end{minipage}
\end{figure}

\textbf{Ablation of Scaled \psilora.}
In order to further stress-test Scaled \psilora
we test various hyperparameters on GLUE-MNLI given a fixed $(\beta_1, \beta_2) = (0.9, 0.99)$.
We compare Scaled \psilora with metric powers $\metricpow \in \{0.25, 0.5\}$,
use a K-FAC vs. Shampoo statistics, and consider damping $\delta \in \{10^{-1}, 10^{-3}, 10^{-5}, 10^{-8}\}$.
Results are shown in \cref{fig:ablation_scaled_psilora_glue_mnli}.
We observe that K-FAC performs best with $\metricpow = 0.5$ and lower damping $\delta \leq 10^{-5}$,
and it is quite robust to the choice of $\delta$.
Shampoo statistics perform well when $\metricpow = 0.25$,
but we do not see a clear advantage over K-FAC.
Most interestingly, we see that Scaled \psilora can perform well for a surprisingly wide range of learning rates,
ranging from $0.01$ to $10.0$.
This supports the claim that Scaled \psilora decreases the burden of learning rate tuning significantly.

\begin{figure}[ht!]
    \centering
    \includegraphics[width=0.45\textwidth]{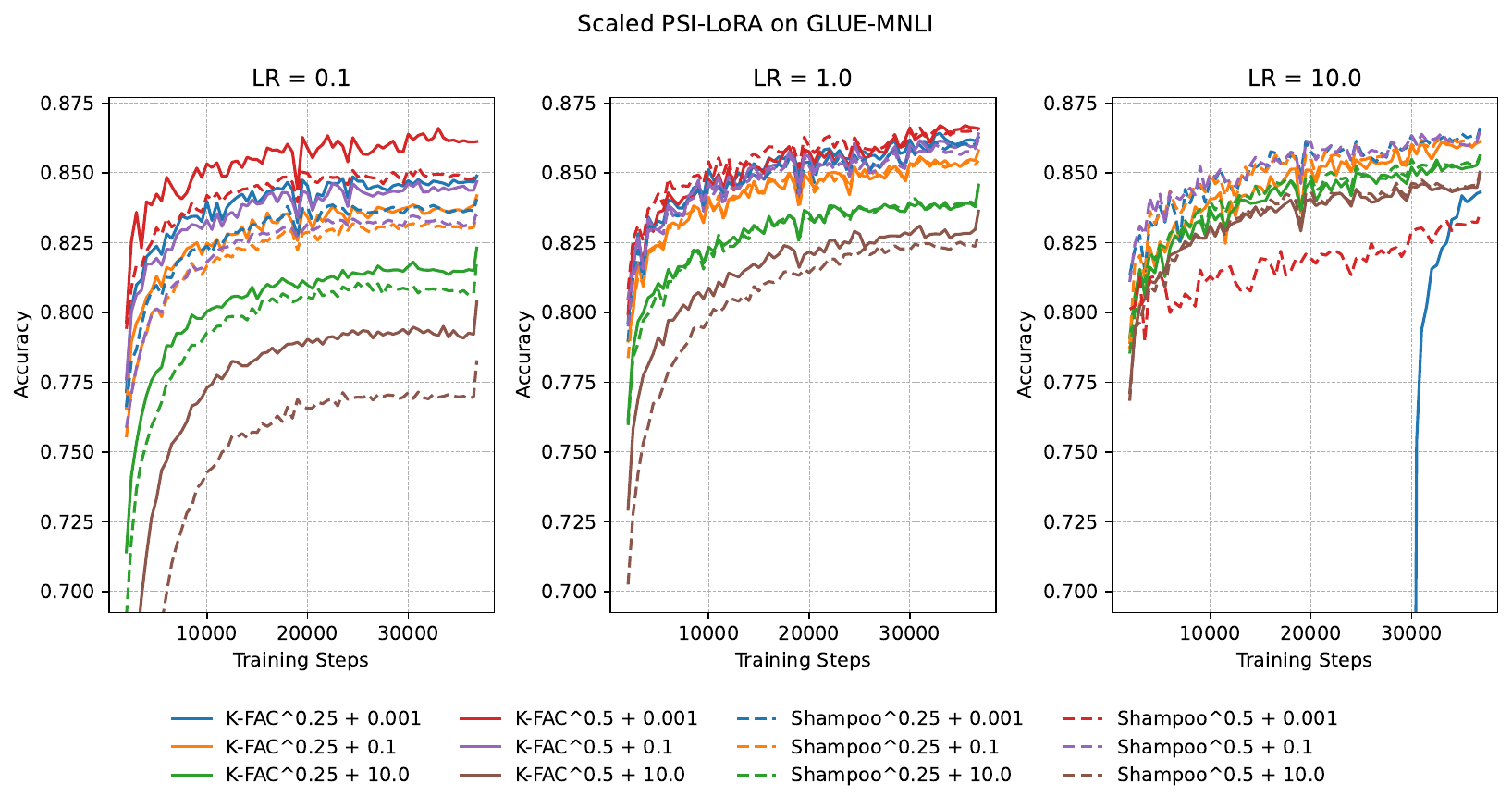}
    \includegraphics[width=0.45\textwidth]{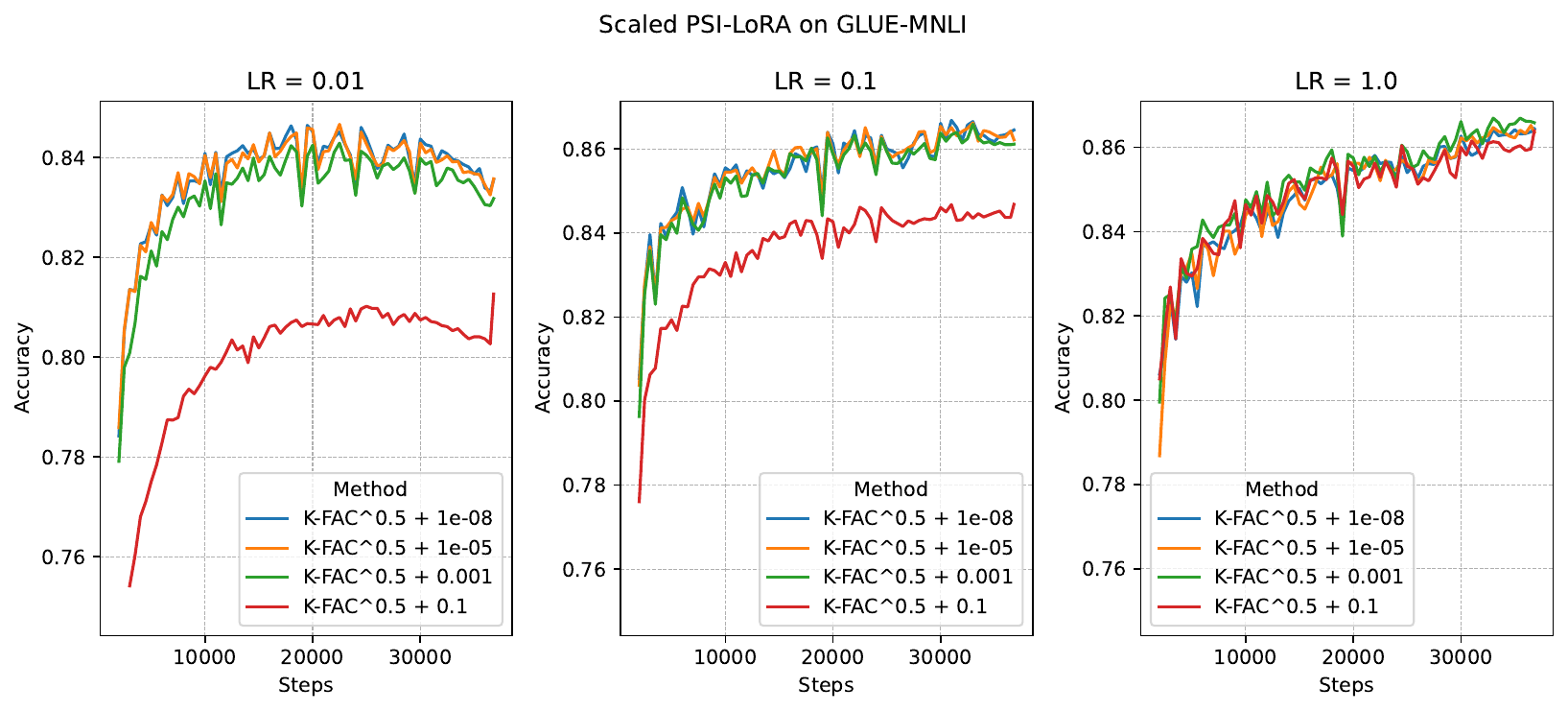}
    \caption{Ablation of Scaled \psilora on GLUE-MNLI.}
    \label{fig:ablation_scaled_psilora_glue_mnli}
\end{figure}

\section{Discussion and Limitations}
\label{sec:discussion}
\psilora trades extra per-step compute for avoiding full-matrix SVD projections,
while remaining robust across a wide range of learning rates and hyperparameters.
The main limitations are:
(i) overhead incurred from caching inputs and gradients w.r.t\ outputs and running inner iterations,
(ii) larger $K$ can degrade under stochastic gradients, even on simple tasks (\cref{fig:linear}), but this is also true for the \svdlora oracle,
and (iii) diagonal Kronecker metrics are memory-efficient but crude.

We have experimented with several directions to improve \psilora and \scaledlorsum that we leave to future work:
(i) using low-rank metrics instead of diagonal and updating them with \lorsum,
(ii) memory-efficient gradient accumulation with $\*X$ and $\*S$ and sketching using \lorsum,
(iii) mitigating the stochastic sensitivity with variance reduction \citep{cutkosky2019momentumbasedvariancereduction} or with more principled momentum schemes \citep{beck2009fast},
and (iv) reducing runtime with a distributed implementation \citep{shi2023pytorchshampoo}.


\section*{Impact Statement}
This paper presents work whose goal is to advance the field of Machine Learning.
There are many potential societal consequences of our work, none which we feel must be specifically highlighted here.


\bibliography{ref}
\bibliographystyle{icml2025/icml2025}

\newpage
\appendix
\onecolumn

\section{Additional implementation details}
\label{app:impl}

\subsection{\lorsum technical details}

\paragraph{Caching activations for inner iterations.}
To support $K>1$ inner iterations, we cache the exact layer inputs $\*X_t$ and output gradients $\*S_t$
that define the per-layer gradient $\*G_t=\*S_t^\top\*X_t$.
This is required because the inner iterations access products of the form
$\*G_t\*V^{(k)}=\*S_t^\top(\*X_t\*V^{(k)})$ and $\*G_t^\top\*U^{(k)}=\*X_t^\top(\*S_t\*U^{(k)})$ for updated factors.
With dropout, gradient clipping, loss scaling, or gradient accumulation,
$\*X_t$ and $\*S_t$ must correspond to the post-dropout / post-scaling quantities used in the backward pass.
Otherwise, recomputed products such as $\*G_t\*V^{(k)}$ become inconsistent with the gradient that autograd applied.
In our implementation, we apply clipping and loss-scaling consistently
to the cached $\*X_t$ and $\*S_t$, used by inner iterations and metric updates.

\paragraph{Mixed precision and loss scaling.}
When using AMP with dynamic loss scaling (e.g.\ \texttt{GradScaler}),
$\*S_t$ is produced in scaled units.
Before using cached activations for inner iterations or updating curvature statistics,
we unscale $\*S_t$ by the current scale factor so that $\*G_t=\*S_t^\top\*X_t$ matches the true gradient.
Equivalently, one can keep $\*S_t$ in scaled units \emph{as long as} the same convention is used consistently
for (i) the update term and (ii) any metric / second-moment statistics derived from $\*S_t$.
With gradient accumulation, $\*S_t$ should reflect the same averaging convention
as the final accumulated gradient (e.g., mean over accumulated mini-batches).

\paragraph{Choosing $\rho$ and $\delta$.}
We found that it is best to tune a reparameterization of the proximal parameter $\theta = \rho / \eta$.
In our experiments, we found that $\theta \in [10^{-3}, 10^{-2}]$
to be a good range across tasks and optimizers.
We also consider $\rho = \theta \eta_t$, where $\eta_t$ follows a learning rate schedule,
which also works even as $\eta_t \to 0$, but we set a lower bound on $\rho > 10^{-5}$ to avoid ill-conditioning in the inner systems.
Compare this with AdamW's weight decay \citep{loshchilov2019decoupledweightdecay}, which is also scaled by the learning rate.

We also use $\delta$ for metric damping in scaled variants, which is an important recipe for the stabilizing the K-FAC optimizer.
However, we found that our algorithm, with $\metricpow=1/2$, is surprisingly resilient to low damping values,
which is not the case when $\metricpow=1$.
We show the plots of the ablation experiments in \cref{fig:ablation_scaled_psilora_glue_mnli}.

\paragraph{Solving the inner systems.}
Since all solves / inverses involve symmetric positive definite matrices,
we use Cholesky factorization followed by triangular solves.
In case of failures, the solver falls back to more general and stabler solves.
However, we rarely observed such failures in practice that are not due to bad hyperparameter settings.
In general, we solve the $r\times r$ systems in FP32 without inversion.

\subsection{Gradient clipping}
\label{sec:gradient-clipping}

When using gradient clipping by global norm,
we considered applying the clipping factor to either $\*X_t$, $\*S_t$, split the scale evenly across both.
Importantly, we update curvature statistics with the \emph{clipped} factors $c^{a}\*X_t$ and $c^{1-a}\*S_t$ for some $a \in [0,1]$,
which noticeably improved performance.
This suggests that gradient clipping's usefulness is in robustifying samples of $\*X_t$ and $\*S_t$
and not only bounding updates within a trust-region.

We found that setting $a=1$, i.e., applying the clipping factor fully to $\*X_t$, worked best in practice,
which made sense since we found that $\|\*X_t\| \gg \|\*S_t\|$ typically
(even after accounting for the $1/B$ scaling in $\*S_t$ from averaging the loss).
However, this is also surprising because the gradient clipping, which we found to be crucial for stable training,
almost cancels out in the final update in when disabling EMA ($\beta_1=\beta_2=0$) and using $\metricpow=1/2$.

To see this, let us consider the clipped inputs $c \*X_t$ for $c \leq 1$
and set $\delta \approx 0$ and $\metricpow=1/2$.
Next, set the K-FAC statistics with these clipped inputs so that the metrics become
\begin{equation*}
    \*D_\*U = (\*S_t^\top \*S_t)^{1/2},
    \qquad
    \*D_\*V = c(\*X_t^\top \*X_t)^{1/2},
\end{equation*}
Take the \scaledlorsum iterations \cref{eq:metric-psilora-alternating-update}.
The shrinkage term will clearly cancel out the scaling in both updates.
However, taking the effective preconditioned steps $\tilde{\*G}_{\*U}$ and $\tilde{\*G}_{\*V}$, we have
\begin{equation*}
    \begin{aligned}
        \tilde{\*G}_{\*U}
        &=
        (\*S_t^\top \*S_t)^{-1/2} \*G_t \*V
        \big(
            \*V^\top (\*X_t^\top \*X_t)^{1/2} \*V + \rho {\color{Red} c^{-1}} \*I
        \big)^{-1}
        \\
        \tilde{\*G}_{\*V}
        &=  
        (\*X_t^\top \*X_t)^{-1/2} \*G_t^\top \*U
        \big(
            \*U^\top (\*S_t^\top \*S_t)^{1/2} \*U + \rho \*I
        \big)^{-1}
        .
    \end{aligned}
\end{equation*}
This interesting scale-invariance property follows from the half-Fisher preconditioning using K-FAC statistics,
which is easy to prove, as we show in \cref{app:sqrt-kfac}.
However, there is also an implicit increase of the proximal term by $c^{-1} \geq 1$ in the $\*U$ update,
which may offer a different perspective on choosing $\rho$ adaptively based on the gradient norm
since $c^{-1} \propto \| \nabla f(\*W_t)\| $ when it is active.

Indeed, such adaptive damping strategies have long been studied
under the lens of differential equations as Hessian-driven damping
\citep{alvarez2000minimizingpropertysecondorder,attouch2014dynamicalinertialforwardbackward,su2016differentialequationnesterovaccelerated},
and recently in the context of adaptive optimization methods but with $\sqrt{\|\nabla f(\*W_t)\|}$ instead
\citep{zhang2021preconditioned,mishchenko2023regularizednewton}.
However, we are not aware of connections between Hessian damping and gradient clipping in deep learning.

In general, our exact method does not fully cancel out the clipping effect
since typically we have $\beta_1, \beta_2 \geq 0.9$.
Nonetheless, the above analysis suggests that clipping has a more subtle effect
than just bounding the step size.

\subsection{Distributed data-parallel training}
\label{sec:ddp}

Let us assume that the loss is a mean over $B$ samples per step.
In data-parallel training with $R$ workers,
the effective step should then use the global mean gradient
\begin{equation*}
    \*G_t = \frac{1}{R}\sum_{r=1}^R \*G_{r,t},
    \qquad
    \*G_{r,t} = \*S_{r,t}^\top \*X_{r,t},
\end{equation*}
where $\*X_{r,t}$ and $\*S_{r,t}$ are the local inputs and output gradients on worker $r$.
Standard LoRA updates only require $\*G_t\*V_t$ and $\*G_t^\top\*U_t$,
which autograd computes and DDP synchronizes automatically.
However, our inner iterations require products with \emph{updated} factors,
$\*G_t\*V^{(k)}$ and $\*G_t^\top\*U^{(k)}$, which are not part of the autograd graph.

To obtain consistent inner-iteration updates across workers,
we compute these products locally and all-reduce the resulting thin matrices
\begin{equation*}
    \*G_t\*V^{(k)}
    =
    \frac{1}{R}\sum_{r=1}^R (\*S_{r,t})^\top(\*X_{r,t}\*V^{(k)}),
    \qquad
    \*G_t^\top\*U^{(k)} \;=\;
    \frac{1}{R}\sum_{r=1}^R (\*X_{r,t})^\top(\*S_{r,t}\*U^{(k)}).
\end{equation*}
Both all-reduced quantities have shapes $d_\mathrm{out}\times r$ and $d_\mathrm{in}\times r$,
so the communication cost is modest compared to synchronizing full gradients.
Note that this applies to \lorsum momentum updates as well.

\paragraph{Synchronizing diagonal Kronecker statistics.}
For scaled variants using diagonal Kronecker factors (e.g.\ diagonal K-FAC/Shampoo-style),
we also all-reduce the per-step diagonal statistics
(e.g.\ $\mathrm{diag}(\tfrac{1}{B}\*X_{r,t}^\top\*X_{r,t})$ and $\mathrm{diag}(\tfrac{1}{B}\*S_{r,t}^\top\*S_{r,t})$)
before updating the EMA.
Without this step, each worker solves a different projection subproblem and parameters may drift.

\begin{table}[t]
\centering
\caption{
    Per-layer optimizer state and dominant per-step costs (adapter rank $r$, effective batch size $B$).
    We assume that $r_m \leq r$ for \psilora momentum rank.
}
\label{tab:costs}
\begin{tabular}{lccc}
\toprule
Method & Extra state & Step cost & Dense \\
\midrule
LoRA + SGD & $r(d_\text{in}+d_\text{out})$ & $O((d_\text{in}+d_\text{out})r)$ & No \\
LoRA + AdamW & $2r(d_\text{in}+d_\text{out})$ & $O((d_\text{in}+d_\text{out})r)$ & No \\
LoFT & $2r(d_\text{in}+d_\text{out})$ & $O((d_\text{in}+d_\text{out})r^2)$ & No \\
Precond-LoRA & $r(d_\text{in}+d_\text{out})$ & $O((d_\text{in}+d_\text{out})r^2)$ & No \\
SVDLoRA + AdamW & $2d_\text{in}d_\text{out}$ & $O((d_\text{in}+d_\text{out})^2 r)$ & Yes \\
\psilora & $r_m(d_\text{in}+d_\text{out})$ & $O((d_\text{in}+d_\text{out})r^2 K)$ & No \\
Scaled \psilora & $(r_m+1)(d_\text{in}+d_\text{out})$ & $O((d_\text{in}+d_\text{out})r^2 K)$ & No \\
\bottomrule
\end{tabular}
\end{table}

\section{Proofs}
\label{app:proofs}

\subsection{Proof of \cref{thm:subspace-iteration}}
\label{app:subspace-iteration-proof}

\paragraph{Theorem \ref{thm:subspace-iteration} (restated).}
Consider the iterations \cref{eq:proximal-alternating-updates},
where $\bar{\*W}_{t+1} = \*U_t^{(0)} \*V_t^{(0)} - \eta \*G_t$ and $\rho = 0$.
Assume $\mathspan(\*V^{(0)})$ has nonzero overlap with the top-$r$ right singular subspace of $\bar{\*W}_{t+1}$.
Let $\sigma_1 \geq \ldots \sigma_{\min\{d_\text{in}, d_\text{out}\}}$ be the singular values
of $\bar{\*W}_{t+1}$ and let $\*W_{t+1}^\mathrm{svd} = \Pi_r(\bar{\*W}_{t+1})$ be its best rank-$r$ approximation.
Then, for some constant $C > 0$ that depends on the initialization $\*V_t^{(0)}$, we have that
\begin{equation*}
    \| \*U_t^{(k)} \*V_t^{(k)} -  \*W_{t+1}^\mathrm{svd}\|_\mathrm{F}^2
    \leq C \left( \frac{\sigma_{r+1}}{\sigma_{r}} \right)^{2k}.
\end{equation*}

\begin{proof}
The iterations \cref{eq:psilora-alternating-update} satisfy alternating projections
following \cref{eq:psilora-alternating-projections}, giving
\begin{equation}
    \begin{aligned}
    \mathspan (\*U_{t+1}^{(k+1)})
    = \mathspan (\bar{\*W}_{t+1} \*V_{t+1}^{(k)})
    = \mathspan (\bar{\*W}_{t+1} \bar{\*W}_{t+1}^\top \*U_{t+1}^{(k)}),
    \end{aligned}
\end{equation}
which is a block power iteration on $\bar{\*W}_{t+1} \bar{\*W}_{t+1}^\top$.
By symmetry, $\mathspan (\*V_{t+1}^{(k+1)}) = \mathspan (\bar{\*W}_{t+1}^\top \bar{\*W}_{t+1} \*V_{t+1}^{(k)})$,
which is a block power iteration on $\bar{\*W}_{t+1}^\top \bar{\*W}_{t+1}$.
This is exactly a subspace iteration method \citep{subspaceiteration,gu2014subspaceiterationrandomizationsingular},
so the convergence in Frobenius norm follows.
\end{proof}

\begin{remark}
    In subspace iteration, the iterates are orthonormalized,
    which amounts to taking the inverse square roots in \cref{eq:proximal-alternating-updates}.
    In other words, the multiplications $\bar{\*W}_{t+1} \bar{\*W}_{t+1}^\top$ and $\bar{\*W}_{t+1}^\top \bar{\*W}_{t+1}$ are explicitly done,
    whereas they are implicitly done via alternating ALS updates.
\end{remark}
\begin{remark}
    \label{rem:initialization}
    The assumption required for $\mathspan(\*V^{(0)})$ is mild in practice since we warm-start from the previous step's factors.
    Also, at LoRA initialization $\*V_0$ is a Gaussian matrix, which satisfies this assumption almost surely.
    This also suggests that the order we adopt in \cref{eq:psilora-alternating-update,eq:proximal-alternating-updates}
    is preferable for LoRA training, and that random initializations for $\*V$ can be used in other applications of \lorsum.
\end{remark}

\section{Interpretations of the proximal projection}
\label{app:proximal-interpretation}

This appendix records two complementary ways to view the proximal subproblem in \cref{eq:proximal-subproblem}.

\paragraph{(i) A stabilized approximation to the rank-$r$ projection.}
Let $\bar{\*W}_{t+1}=\*W_t-\eta\*G_t$ denote the full step before projection.
The proximal factor projection solves
\begin{equation*}
    \min_{\*U,\*V} \
    \frac{1}{2} \|\*U\*V^\top - \bar{\*W}_{t+1}\|_\mathrm{F}^2
    + \frac{\rho}{2} \big( \|\*U - \*U_t\|_\mathrm{F}^2 + \|\*V - \*V_t\|_\mathrm{F}^2 \big).
\end{equation*}
This interpretation follows by the construction of the sub-problem as a proximal approximation to the idealized rank-$r$ projection step.
When $\rho=0$ and the inner solve is exact and the objective reduces to
$\min_{\rank(\*W)\le r}\|\*W-\bar{\*W}_{t+1}\|_\mathrm{F}^2$,
whose solution is the truncated SVD projection $\Pi_r(\bar{\*W}_{t+1})$, i.e., the \svdlora oracle.
For $\rho>0$ and a finite number of ALS iterations,
the proximal terms select a particular factorization that stays close to $(\*U_t,\*V_t)$,
which improves stability in deep networks while still tracking the same rank-$r$ projection in the limit.

\paragraph{(ii) Lifted-space proximal step under a non-linear constraint.}
Define the lifted variable $\*z=\mathrm{vec}(\*W,\sqrt{\rho}\*U,\sqrt{\rho}\*V)$
and $\*z_t=\mathrm{vec}(\*W_t,\sqrt{\rho}\*U_t,\sqrt{\rho}\*V_t)$.
Then, the above problem is the proximal map of the linearized model of $f$ in the lifted space,
with the non-linear constraint $\*W=\*U\*V^\top$ enforced by restricting to feasible $(\*W,\*U,\*V)$.
We can write the proximal step as
\begin{equation*}
    \min_{\*z} \
    \langle \nabla_\*W f(\*z_t), \*W - \*W_t \rangle
    + \frac{1}{2\eta} \|\*z - \*z_t\|_2^2
    + \mathbf{1}_{\*W=\*U\*V^\top}(\*z)
    ,
\end{equation*}
where $\mathbf{1}_E(\*z)$ is the indicator function of the non-convex constraint $ E$ on $\*z$.
This interpretation connects to proximal methods for constrained optimization \citep{parikh2014proximalalgorithms},
where the non-convex constraint $\*W=\*U\*V^\top$ is handled implicitly by solving over the factorized variables.
This viewpoint also suggests possible extensions to other constraints on $\*W$ (e.g., sparsity, quantization) via proximal terms
in the lifted space, handling $f(\*z)$ directly without linearizing over $\*W$,
or considering relaxations of the non-convex constraint.

\section{Square-root K-FAC yields implicit gradient normalization}
\label{app:sqrt-kfac}

This section describes the property that \emph{half-Fisher} (square-root K-FAC) preconditioning
yields implicit gradient normalization in the single-layer case.

Let $\*G = \*S^\top\*X$ be a single-layer gradient.
Set the metrics $\*D_\*U = (\*S^\top\*S)^{1/2}$ and $\*D_\*V = (\*X^\top\*X)^{1/2}$.
Without loss of generality, set damping $\delta=0$, and consider the half-K-FAC metric update
\begin{equation*}
    \*W
    \gets
    \*W -\eta \*D_\*U^{-1} \*G \*D_\*V^{-1}
    =
    \*W - \eta (\*S^\top\*S)^{-1/2} \*S^\top \*X (\*X^\top\*X)^{-1/2}
    =
    \*W - \eta \hat{\*S}^\top \hat{\*X}
    ,
\end{equation*}
where we defined the orthonormalized factors
\begin{equation*}
    \hat{\*S} = \*S (\*S^\top\*S)^{-1/2},
    \qquad
    \hat{\*X} = \*X (\*X^\top\*X)^{-1/2},
\end{equation*}
so that $\hat{\*S}^\top \hat{\*S} = \*I$ and $\hat{\*X}^\top \hat{\*X} = \*I$.
This immediately implies that
\begin{equation*}
    \|\hat{\*S}\|_\mathrm{F}^2 = \rank(\hat{\*S}) = \rank(\*S),
    \qquad
    \|\hat{\*X}\|_\mathrm{F}^2 = \rank(\hat{\*X}) = \rank(\*X).
\end{equation*}
Furthermore, define the preconditioned gradient $\tilde{\*G} = \hat{\*S}^\top \hat{\*X}$.
Its squared Frobenius norm can be expressed as
\begin{equation*}
    \|\tilde{\*G} \|_\mathrm{F}^2
    =
    \tr(\hat{\*X}^\top \hat{\*S} \hat{\*S}^\top \hat{\*X})
    =
    \tr(\mathcal{P}_{\*S} \mathcal{P}_{\*X})
    \leq
    \min\{\rank(\*S), \rank(\*X)\}
    ,
\end{equation*}
where, typically, $\min\{\rank(\*S), \rank(\*X)\} \leq B$ for many modern linear layers
(since most layers have one side dimension larger than the batch size $B$).
This shows that half-metric (square-root) K-FAC preconditioning controls the step norm by subspace overlap,
rather than by the raw magnitudes of $\*X$ and $\*S$.

This property does not necessarily apply to the \scaledlorsum updates in \cref{eq:metric-psilora-alternating-update}
due to the way the metrics are applied (assuming $\rho=0$):
\begin{equation*}
    \begin{aligned}
    \*U^{(k+1)}
    &=
    \big(
        \*U \*V^\top \*D_\*V \*V^{(k)}
        + \*D_\*U^{-1} \Delta \*V^{(k)}
    \big)
    \big(
    (\*V^{(k)})^\top \*D_\*V \*V^{(k)}
    \big)^{-1}
    \\
    &=
    \big(
        \*U \*V^\top (\*X^\top \*X)^{1/2} \*V^{(k)}
        + \hat{\*S}^\top \*X \*V^{(k)}
    \big)
    \big(
    (\*V^{(k)})^\top (\*X^\top \*X)^{1/2} \*V^{(k)}
    \big)^{-1}
    ,
    \end{aligned}
\end{equation*}
so the soft normalization property does not hold \emph{exactly} unless
$(\*V^{(k)})^\top (\*X^\top \*X)^{1/2} \*V^{(k)} = \big( (\*V^{(k)})^\top \*X^\top \*X \*V^{(k)}\big)^{1/2}$,
which requires orthogonal $\*V^{(k)}$ and commuting $\*X^\top \*X$ and $\*V^{(k)} (\*V^{(k)})^\top$.
Nonetheless, we still observe stability using large learning rates with \scaledlorsum,
which suggests that the implicit normalization property of K-FAC is still approximately preserved in our updates.

\section{Algorithms and PyTorch reference implementation}
\label{app:algo}

This appendix collects algorithms and a minimal reference implementation in PyTorch \citep{pytorch}.
Note that the code itself does not need PyTorch's autograd for the updates,
but we use it for consistency and ease of integration with existing training pipelines.

In the paper, $\*W=\*U\*V^\top$ with $\*U\in\RR^{d_\mathrm{out}\times r}$ and $\*V\in\RR^{d_\mathrm{in}\times r}$.
In the reference code, we store factors as $(\*V^\top,\*U)$,
which makes sense when read as (\verb|weight_in|, \verb|weight_out|),
where the weights maintain their original shape in PyTorch's convention.
We maintain the same notation for the metrics, for example, $(\*D_\*V, \*D_\*U)$,
which is (\verb|input_metric|, \verb|output_metric|),
which flows naturally from left to right.

\begin{algorithm}[!ht]
\caption{\psilora step with low-rank momentum}
\label{algo:psilora-step}
\begin{algorithmic}[1]
    \STATE \textbf{Input:} current factors $(\*U_t,\*V_t)$, momentum buffer $\mathcal{M}_{t-1}$ (stored as low-rank factors),
    learning rate $\eta$, momentum $\alpha$, \lorsum hyperparameters $(K,\rho)$.
    \STATE Cache during forward: $\*X_t$ (layer inputs). Cache during backward: $\*S_t$ (output gradients).
    \STATE Form low-rank gradient representation: $\*G_t=\*S_t^\top\*X_t$ (unmaterialized).
    \STATE Form full step $\bar{\*W}_{t+1} = \*W_t - \eta(\*G_t+\alpha\mathcal{M}_{t-1})$ (unmaterialized).
    \STATE Project step with \lorsum:
    \begin{equation*}
      \*U_{t+1}\*V_{t+1}^\top
      = \*W_{t+1}
      \gets
      \lorsum(\bar{\*W}_{t+1}; \,K,\rho).
    \end{equation*}
    \STATE Form full momentum $\bar{\mathcal{M}}_{t} = \alpha \mathcal{M}_{t-1} + \*G_t$ (unmaterialized).
    \STATE Update low-rank momentum:
    \begin{equation*}
        \mathcal{M}_t \gets \lorsum(\bar{\mathcal{M}}_{t};\,K,\rho)
    \end{equation*}
\end{algorithmic}
\end{algorithm}

\begin{algorithm}[!ht]
\caption{\lorsum (proximal ALS iterations for rank-$r$ projection of a low-rank sum)}
\label{algo:lorsum}
\begin{algorithmic}[1]
    \STATE \textbf{Input:} low-rank factors and coefficients $\{c_j,\*U_j,\*V_j\}_{j=1}^n$,
    warm start $(\*U_1,\*V_1)$, inner iterations $K$, proximal parameter $\rho$.
    \STATE Initialize $\*U^{(0)}\gets \*U_1$, $\*V^{(0)}\gets \*V_1$.
    \STATE Define $\bar{\*W} = \sum_{j=1}^n c_j\*U_j\*V_j^\top$.
    \FOR{$k=0,\dots,K-1$}
        \STATE $\*U^{(k+1)} \gets \big(\bar{\*W}\*V^{(k)}+\rho\,\*U_1\big)\big((\*V^{(k)})^\top\*V^{(k)}+\rho \*I\big)^{-1}$.
        \STATE $\*V^{(k+1)} \gets \big(\bar{\*W}^\top\*U^{(k+1)}+\rho\,\*V_1\big)\big((\*U^{(k+1)})^\top\*U^{(k+1)}+\rho \*I\big)^{-1}$.
    \ENDFOR
    \STATE \textbf{Output:} $(\*U^{(K)},\*V^{(K)})$.
\end{algorithmic}
\end{algorithm}

\subsection{\lorsum and \psilora}
\label{app:lorsum-algo}
\lorsum and \psilora algorithms are summarized in \cref{algo:lorsum,algo:psilora-step}.
A minimal reference implementation of \lorsum is provided in \cref{lst:lorsum},
and an implementation for \scaledlorsum is provided in \cref{lst:scaled-lorsum}.
The algorithms contain dense-matrix notation for clarity (tagged with ``unmaterialized''),
but the implementation only uses thin-matrix operations
and never materializes full $d_\mathrm{out}\times d_\mathrm{in}$ matrices,
as shown in the reference code.

\paragraph{Momentum buffer update.}
We apply momentum at the level of the (approximate) full-rank update and then project back to rank $r_m$.
Concretely, we consider the full update matrix $\bar{\*W}_{t+1}=\*W_t-\eta(\*G_t+\alpha \mathcal{M}_{t-1})$
and project it with \lorsum, while separately updating $\mathcal{M}_t$ as a low-rank approximation of
$\*G_t+\alpha\mathcal{M}_{t-1}$.
This might seem redundant, but it is important since we do not separately project
gradient and momentum but rather together.
In other words, $\bar{\*W}_{t+1}$ is is represented as a sum of low-rank terms,
each of which is at most rank $r$, but collectively can potentially be of higher rank.
\lorsum can project this sum back to rank $r$ directly without materializing the full matrix.
We found this to be an important detail for achieving competitive performance and a distinguishing feature of \psilora.

For simplicity, we always set $r_m = r$, except in the linear task experiment \cref{fig:linear-momentum},
where we explicitly compare momentum rank budgets.

\begin{listing}[!ht]
\caption{\lorsum reference implementation}
\label{lst:lorsum}
\begin{minted}[fontsize=\footnotesize, linenos, breaklines]{python}
def lorsum(
    factors: list[tuple[torch.Tensor, torch.Tensor]],
    coefficients: list[float],
    num_iters: int = 1,
    rho: float = 0.0,
) -> tuple[torch.Tensor, torch.Tensor]:
    """
    factors are stored as (V_i^T, U_i) so that W = U @ V.
      - V_i^T: (r_i, d_in)
      - U_i:   (d_out, r_i)
    """
    assert len(factors) >= 2 and len(factors) == len(coefficients)
    solve = torch.linalg.solve

    Vt, U = factors[0]                    # anchor (V^T, U)
    Vt_k, U_k = Vt.clone(), U.clone()     # warm start
    r = Vt_k.shape[0]
    eye_r = torch.eye(r, device=Vt_k.device, dtype=Vt_k.dtype)

    for k in range(2 * num_iters):
        if k % 2 == 0:
            # update V^T
            rhs = Vt.mul(rho)
            for c_i, (Vt_i, U_i) in zip(coefficients, factors):
                rhs.add_((U_k.T @ U_i) @ Vt_i, alpha=c_i)
            Vt_k = solve(U_k.T @ U_k + eye_r.mul(rho), rhs)
        else:
            # update U
            rhs = U.mul(rho)
            for c_i, (Vt_i, U_i) in zip(coefficients, factors):
                rhs.add_(U_i @ (Vt_i @ Vt_k.T), alpha=c_i)
            U_k = solve(Vt_k @ Vt_k.T + eye_r.mul(rho), rhs.T).T

    return Vt_k, U_k
\end{minted}
\end{listing}

\subsection{\scaledlorsum and scaled \psilora}
\label{app:scaled-lorsum-algo}
Scaled \psilora step (\cref{algo:scaled-psilora-step}) is similar to \psilora step (\cref{algo:psilora-step}),
except that it uses exponential moving averages and \scaledlorsum for the update (not for the momentum).
We fix the hyperparameters $(K,\rho)$ to be the same for both operators for simplicity.

\begin{algorithm}[!ht]
\caption{Scaled \psilora step with low-rank momentum and diagonal K-FAC metric with fractional power}
\label{algo:scaled-psilora-step}
\begin{algorithmic}[1]
    \STATE \textbf{Input:} current factors $(\*U_t,\*V_t)$, low-rank momentum buffer $\mathcal{M}_{t-1}$,
    learning rate $\eta$, EMA smoothing of gradient and metrics $\beta_1$ and $\beta_2$, metric damping $\delta$,
    \lorsum hyperparameters $(K,\rho)$.
    \STATE Cache during forward: $\*X_t$ (layer inputs). Cache during backward: $\*S_t$ (output gradients).
    \STATE Form low-rank gradient representation: $\*G_t = \*S_t^\top\*X_t$ (unmaterialized).
    \STATE Update diagonal K-FAC statistics:
    \begin{equation*}
        \*v_{s,t} \gets \beta_2 \*v_{s,t-1} + (1-\beta_2) \mathrm{diag}\left(\tfrac{1}{B}\*S_t^\top\*S_t\right),
        \qquad
        \*v_{x,t} \gets \beta_2 \*v_{x,t-1} + (1-\beta_2) \mathrm{diag}\left(\tfrac{1}{B}\*X_t^\top\*X_t\right).
    \end{equation*}
    \STATE Form full step $\bar{\*W}_{t+1} = \*W_t - \eta( (1-\beta_1) \*G_t + \beta_1 \mathcal{M}_{t-1})$ (unmaterialized).
    \STATE Form metric factors $\*D_{\*U} = (\*v_{s,t} + \delta)^{\metricpow}$, $\*D_{\*V} = (\*v_{x,t} + \delta)^{\metricpow}$.
    \STATE Project step with \scaledlorsum:
    \begin{equation*}
      \*U_{t+1}\*V_{t+1}^\top
      = \*W_{t+1}
      \gets
      \scaledlorsum(\bar{\*W}_{t+1}; \*D_{\*U} , \*D_{\*V} , K,\rho).
    \end{equation*}
    \STATE Form full momentum $\bar{\mathcal{M}}_{t} = \beta_1 \mathcal{M}_{t-1} + (1-\beta_1) \*G_t$ (unmaterialized).
    \STATE Update low-rank momentum:
    \begin{equation*}
        \mathcal{M}_t \gets \lorsum(\bar{\mathcal{M}}_{t};\,K,\rho)
    \end{equation*}
\end{algorithmic}
\end{algorithm}

\begin{algorithm}[!ht]
\caption{\scaledlorsum (\lorsum with Kronecker-factored metrics)}
\label{algo:scaled-lorsum}
\begin{algorithmic}[1]
    \STATE \textbf{Input:} low-rank factors and coefficients $\{c_j,\*U_j,\*V_j\}_{j=1}^n$
    warm start $(\*U_1,\*V_1)$, metrics $\*D_{\*U}, \*D_{\*V}$, inner iterations $K$, proximal parameter $\rho$.
    \STATE Initialize $\*U^{(0)}\gets \*U_1$, $\*V^{(0)}\gets \*V_1$.
    \STATE Define $\Delta = \sum_{j=2}^n c_j\*U_j\*V_j^\top$ (so that $\tilde{\*W} = c_1\*U_1\*V_1^\top + \*D_{\*U}^{-1} \Delta \*D_{\*V}^{-1}$).
    \FOR{$k=0,\dots,K-1$}
        \STATE $\*U^{(k+1)}
        = \big(c_1 \*U_1 \*V_1^\top \*D_\*V \*V^{(k)} + \rho \*U_1 + \*D_\*U^{-1} \Delta \*V^{(k)}\big)\big((\*V^{(k)})^\top \*D_\*V \*V^{(k)} + \rho \*I_r\big)^{-1}$.
        \STATE $\*V^{(k+1)} = \big(c_1 \*V_1 \*U_1^\top \*D_\*U \*U^{(k+1)} + \rho \*V_1 + \*D_\*V^{-1} \Delta^\top \*U^{(k+1)}\big)\big((\*U^{(k+1)})^\top \*D_\*U \*U^{(k+1)} + \rho \*I_r\big)^{-1}$.
    \ENDFOR
    \STATE \textbf{Output:} $(\*U^{(K)},\*V^{(K)})$.
\end{algorithmic}
\end{algorithm}

\paragraph{Partially preconditioned low-rank sums.}
Note that, in \scaledlorsum (\cref{algo:scaled-lorsum}),
The low-rank sum is expressed as
$\tilde{\*W} = c_1\*U_1\*V_1^\top + \*D_{\*U}^{-1} \big( \sum_{j=2}^n c_j\*U_j\*V_j^\top \big) \*D_{\*V}^{-1}$,
where the first term is the warm start and the remaining terms are preconditioned by the Kronecker factors
and constitute the effective step $\Delta = \sum_{j=2}^n c_j\*U_j\*V_j^\top$.
In general, there is no reason to choose only one pair of factors such that $\*W_t = \*U_1\*V_1^\top$,
but we maintain this convention as we have not found a useful case where $\*W_t$ itself is represented by multiple terms.
However, this is quite plausible if we are considering a combination of previous iterates or multiple LoRA adapters in the update.

\begin{listing}[!ht]
\caption{\scaledlorsum reference implementation}
\label{lst:scaled-lorsum}
\begin{minted}[fontsize=\footnotesize, linenos, breaklines]{python}
def f_lorsum(
    factors: list[tuple[torch.Tensor, torch.Tensor]],
    coefficients: list[float],
    metrics: tuple[torch.Tensor, torch.Tensor],
    num_iters: int = 1,
    rho: float = 0.0,
) -> tuple[torch.Tensor, torch.Tensor]:
    """
    factors are stored as (V_i^T, U_i) so that W = U @ V.
    metrics are stored as (D_V, D_U).
      - V_i^T: (r_i, d_in)
      - U_i:   (d_out, r_i)
      - D_V:   (d_in,)
      - D_U:   (d_out,)
    """
    assert len(factors) >= 2 and len(factors) == len(coefficients)
    solve = torch.linalg.solve

    Vt, U = factors[0]                    # anchor (V^T, U)
    Vt_k, U_k = Vt.clone(), U.clone()     # warm start
    r = Vt_k.shape[0]
    eye_r = torch.eye(r, device=Vt_k.device, dtype=Vt_k.dtype)

    # Precompute preconditioned step factors (dummy for the anchor)
    D_V, D_U = metrics
    prec_factors = [None] + [
        (Vt * D_V.reciprocal()[None, :], D_U.reciprocal()[:, None] * U)
        for Vt, U in factors[1:]
    ]

    for k in range(int(2 * num_iters)):
        if k % 2 == 0:
            # update V^T
            scaled_U_k = D_U[:, None] * U_k
            rhs = Vt.mul(rho)
            rhs.add_((scaled_U_k.T @ U) @ Vt, alpha=coefficients[0])
            for j in range(1, len(coefficients)):
                prec_Vt_j, _ = prec_factors[j]
                _, U_j = factors[j]
                rhs.add_((U_k.T @ U_j) @ prec_Vt_j, alpha=coefficients[j])
            Vt_k = solve(scaled_U_k.T @ U_k + eye_r.mul(rho), rhs)
        else:
            # update U
            scaled_Vt_k = Vt_k * D_V[None, :]
            rhs = U.mul(rho)
            rhs.add_(U @ (Vt @ scaled_Vt_k.T), alpha=coefficients[0])
            for j in range(1, len(coefficients)):
                _, prec_U_j = prec_factors[j]
                Vt_j, _ = factors[j]
                rhs.add_(prec_U_j @ (Vt_j @ Vt_k.T), alpha=coefficients[j])
            U_k = solve(Vt_k @ scaled_Vt_k.T + eye_r.mul(rho), rhs.T).T

    return Vt_k, U_k
\end{minted}
\end{listing}

\paragraph{Diagonal Kronecker-factored metrics with fractional power}
Scaled \psilora replaces the Euclidean proximal terms in \cref{eq:proximal-subproblem} with Kronecker-structured metrics.
In practice we use diagonal Kronecker factors derived from running averages of second moments,
e.g.\ diagonal K-FAC statistics:
\begin{equation*}
    \*v_{s,t} = \mathrm{diag} \left(\tfrac{1}{B}\*S_t^\top\*S_t\right),
    \qquad
    \*v_{x,t} = \mathrm{diag} \left(\tfrac{1}{B}\*X_t^\top\*X_t\right),
\end{equation*}
where we initialize $\*v_{s,0}=\*1$ and $\*v_{x,0}=\*1$.

In practice, we maintain EMA estimates of $\*v_{s,t}$ and $\*v_{x,t}$ and then form \emph{fractional} metric factors with damping
\begin{equation*}
    \*D_{\*U,t} := (\*v_{s,t}+\delta)^{\metricpow},
    \qquad
    \*D_{\*V,t} := (\*v_{x,t}+\delta)^{\metricpow},
\end{equation*}
and pass $(\*D_{\*U,t},\*D_{\*V,t})$ into the \scaledlorsum routine.

We also considered diagonal Shampoo-style statistics \citep{gupta2018shampoopreconditionedstochastictensor} in the ablation experiments,
\begin{equation*}
    \*v_{s,t}^\text{shampoo} = \mathrm{diag} \left(\tfrac{1}{B}\*S_t^\top (\*X_t\*X_t^\top) \*S_t\right),
    \qquad
    \*v_{x,t}^\text{shampoo} = \mathrm{diag} \left(\tfrac{1}{B}\*X_t^\top (\*S_t\*S_t^\top) \*X_t\right),
\end{equation*}
We have found Shampoo to be marginally better (when $\metricpow = 1/4$) but more expensive to compute.
It requires two additional thin GEMM operations per step to compute $\*X_t\*X_t^\top$ and $\*S_t\*S_t^\top$,
which cost $O(B^2 (d_\mathrm{in} + d_\mathrm{out}))$.
This can be prohibitive when the \emph{effective} batch size $B$ is large,
especially for long sequences in language models, so we decided to use K-FAC to maintain simplicity and efficiency.

\section{Additional experimental details}
\label{app:exp}

\begin{figure}[t]
    \centering
    \includegraphics[width=0.99\linewidth]{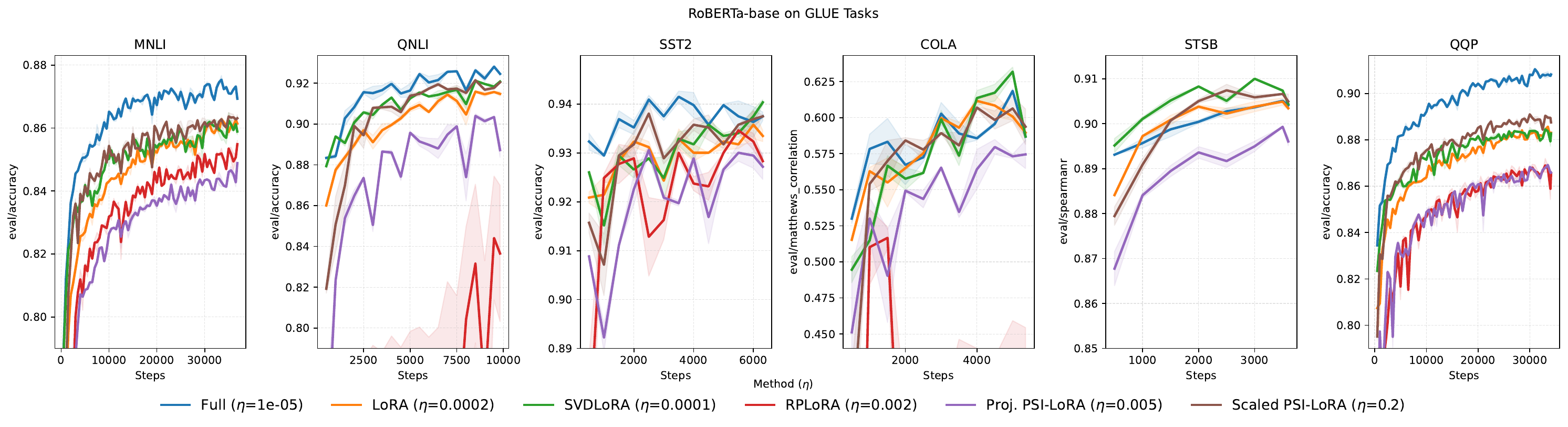}
    \caption{
        Learning curves of validation main metrics for RoBERTa-base on GLUE tasks.
        Legend shows learning rates for each method.
    }
    \label{fig:glue_roberta}
\end{figure}

\subsection{Linear-task momentum ablation}
\label{app:linear-momentum}
We ablate the effect of increasing the rank budget of the low-rank momentum buffer in \cref{fig:linear-momentum}.
We use the same setup as in the linear task experiment in \cref{fig:linear},
with LoRA rank $r=8$ and vary the momentum rank $r_m \in \{r,2r,4r\}$.
We observe that increasing the momentum rank makes \psilora closer to the \svdlora oracle,
which suggests that a higher-rank momentum buffer can better capture the history of updates.
However, this also means that \psilora would plateau faster, as the \svdlora oracle also plateaus.

\paragraph{Linear task tuning notes.}
For LoRA, we typically found the scaling $\eta=0.02$ to work best.
For \svdlora in the stochastic setting, we tuned $\eta$ over the grid $\{1, 2, 5\} \times 10^{i}$ for $i \in \{-2, -1, 0, 1\}$ and momentum $\alpha \in \{\dots, 0.7, 0.75, 0.8, 0.85, 0.9, 0.95\}$,
and found $\eta=0.1$ and $\alpha=0.75$ to perform best.

\begin{figure}[!ht]
    \centering
    \includegraphics[width=0.5\linewidth]{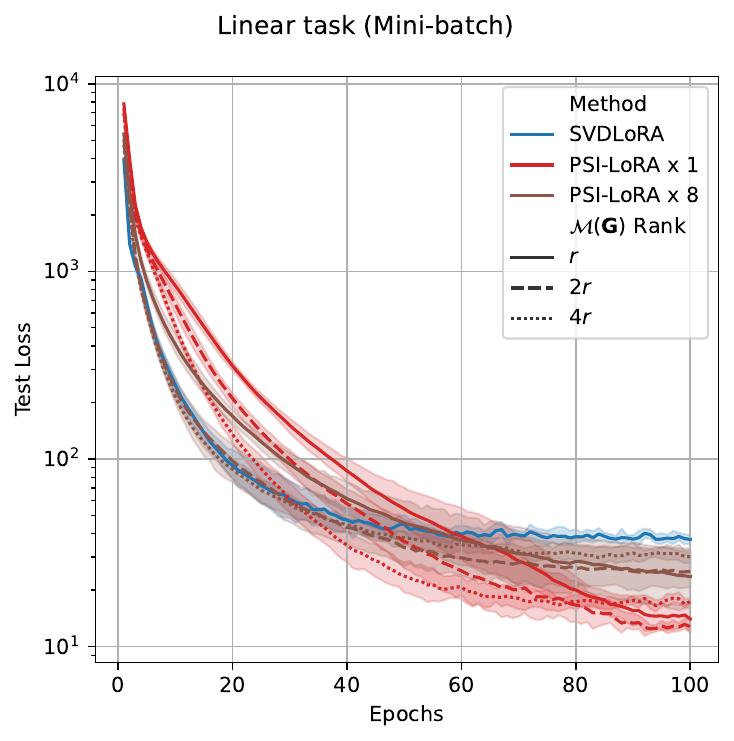}
    \caption{
    Stochastic linear task: increasing the rank budget of the low-rank momentum buffer improves the approximation quality
    and moves \psilora closer to the \svdlora oracle.
    }
    \label{fig:linear-momentum}
\end{figure}

\paragraph{GLUE RoBERTa validation curves.}
We show in \cref{fig:glue_roberta} the full validation curves for RoBERTa-base GLUE experiments
corresponding to \cref{tab:glue_roberta_best_metrics} in the main text.
The learning rates are shown in the legend for each method.

\paragraph{Hyperparameters.}
Unless stated otherwise, LoRA rank is $r=8$ (CIFAR-100 uses $r=16$).
We tune learning rates per method on a log grid $\{1e-4, 2e-4, 5e-4, 1e-3, 2e-3, 5e-3, \ldots \}$
where the range depends on the method.
For example, our method typically uses much higher learning rates in $\{0.1, 0.2, 0.5, 1.0 \}$,
whereas LoRA typically uses $\{1e-4, 2e-4, 5e-4, 1e-3\}$, and full training uses $\{1e-4, 2e-5, 5e-5, 1e-4\}$.
However, we tune over 3 learning rates (typically, the larger ones) per method due to compute constraints.
We average over 3 seeds and report mean and standard deviation.

\end{document}